%% file: clear2022.tex
\documentclass[final,12pt]{clear2022} % Include author names

\input{math_commands.tex}

\usepackage{bm}
\usepackage{mathtools}
\usepackage{paralist}
\usepackage{xcolor}

\SetCommentSty{mycommfont}
\usepackage{booktabs}
\usepackage{caption}
\usepackage{wrapfig}
\usepackage{centernot}
\usepackage{enumitem}   % added by KF

\newcommand{\independent}{\raisebox{0.05em}{\rotatebox[origin=c]{90}{$\models$}}}

\DeclareMathOperator{\diag}{diag}

\DeclareMathOperator{\bern}{Bern}
\DeclareMathOperator{\logi}{Logi}
\newcommand{\blambda}{\bm\lambda}

\newcommand{\btheta}{\bm\theta}
\newcommand{\bphi}{\bm\phi}
\newcommand{\bgamma}{\bm\gamma}
\newcommand{\bbeta}{\bm\beta}
\newcommand{\beps}{\bm\epsilon}

\newcommand{\inv}{^{-1}}
\newcommand{\st}{^*}
\newcommand{\OS}{\mathbb{O}}
\newcommand{\PS}{\mathbb{P}}

\newcommand{\CF}{\mathbb{C}}

\newcommand{\data}{\mathcal{D}}
\newcommand{\x}{\rvx}
\newcommand{\y}{\rvy}
\newcommand{\z}{\rvz}
\newcommand{\rconf}{\rvu}
\newcommand{\conf}{\vu}
\newcommand{\condpriorparam}{p_{\bm\lambda}(\vz|\vx,t)}
\newcommand{\decoderparam}{p_{\vf}(\vy|\vz,t)}
\newcommand{\encoderparam}{q_{\bphi}(\vz|\vx,\vy,t)}

\newcommand{\encoder}{q(\vz|\vx,\vy,t)}

\newcommand{\vaepost}{p(\vz|\vx,\vy,t)}
\newcommand{\trueobs}{p(\vy|\vx,t)}

\newcommand{\binset}{\{0,1\}}

\setlist{nosep} % enum
\AtBeginDocument{%
\setlength{\belowdisplayskip}{4pt plus 1.0pt minus 2.0pt} \setlength{\belowdisplayshortskip}{2pt plus 1.0pt minus 2.0pt}
\setlength{\abovedisplayskip}{5pt plus 1.0pt minus 2.0pt} \setlength{\abovedisplayshortskip}{3pt plus 1.0pt minus 2.0pt}
}
\makeatletter
\def\thm@space@setup{%
  \thm@preskip=3.0pt plus 1.0pt minus 2.0pt
  \thm@postskip=2.0pt plus 1.0pt minus 2.0pt
}
\makeatother

\title[Towards Principled Causal Effect Estimation by Deep Identifiable Models]{Towards Principled Causal Effect Estimation \\ by Deep Identifiable Models}
\usepackage{times}
\clearauthor{%
 \Name{Pengzhou Wu} \Email{wu.pengzhou@ism.ac.jp}\\
 \addr Department of Statistical Science, The Graduate University for Advanced Studies
 \AND
 \Name{Kenji Fukumizu} \Email{fukumizu@ism.ac.jp}\\
 \addr The Institute of Statistical Mathematics
}

\begin{document}

\maketitle

\begin{abstract}
As an important problem in causal inference, we discuss the estimation of treatment effects (TEs). Representing the confounder as a latent variable, we propose Intact-VAE, a new variant of variational autoencoder (VAE), motivated by the prognostic score that is sufficient for identifying TEs. Our VAE also naturally gives representations balanced for treatment groups, using its prior. Experiments on (semi-)synthetic datasets show state-of-the-art performance under diverse settings, including unobserved confounding. Based on the identifiability of our model, we prove identification of TEs under unconfoundedness, and also discuss (possible) extensions to harder settings.
% This paves the way towards principled causal effect estimation by deep neural networks.
\end{abstract}

\section{Introduction}

% of science including epidemiology, education, political science, and economics. Researchers in those fields are increasingly interested in applying machine learning methods, taking advantage of the recent abundant electronic record keeping \citep{kreif2019machine, richens2020improving, hernan2020causal}. Among machine learning methods, there are also growing developments for causal inference and causality \citep{scholkopf2019causality, peters2017elements}, some of which pursue approaches by deep learning \citep{bengio2019meta, suter2019robustly}. 

Causal inference \citep{imbens2015causal, pearl2009causality}, i.e, estimating causal effects of interventions, is a fundamental task across many domains. We address the estimation of treatment effects (TEs), such as effects of public policies or new drugs, based on a set of observations consisting of binary labels for treatment/control (non-treated), outcome, and other covariates. The fundamental difficulty of causal inference is that we never observe \textit{counterfactual} outcomes, which would have been if we had made another decision (treatment or control). While the ideal protocol for causal inference is randomized controlled trials (RCTs), they often have ethical and practical issues, or suffer from expensive costs. Thus, causal inference from observational data is important. 
It introduces other challenges, however. The most crucial one is \textit{confounding}: there may be variables (called confounders) that causally affect both the treatment and the outcome, and spurious correlation follows. 
% We refer readers to \cite{yao2020survey} for an up-to-date survey on causal inference, covering both classical and recent machine learning based methods.

A majority of works in causal inference rely on the \textit{unconfoundedness}, which means that appropriate covariates are collected so that the confounding can be controlled by conditioning on those variables. 
% That is, there is in essence no unobserved confounders.
This is still challenging, due to systematic \textit{imbalance} (difference) of the distributions of the covariates between the treatment and control groups. Classical ways to deal with imbalance are matching and re-weighting \citep{Stuart2010matching,rosenbaum2020modern}. There are semi-parametric methods \citep[e.g. TMLE,][]{van2018targeted}, which have better finite sample performance, and also non-parametric methods \citep[e.g.~CF,][]{wager2018estimation}. Notably, there is a recent rise of interest in learning balanced representation of covariates, which is independent of treatment groups, starting from \cite{johansson2016learning}. % , chernozhukov2018double chipman2010bart,

There are a few lines of works that address the difficult but important problem of \textit{unobserved confounding}.
% , where the methods in the previous paragraph fundamentally does not work. 
Without covariates to adjust for, the naive regression with observed variables introduces bias, if the decision of treatment and the outcome are confounded, as explained in Sec.~\ref{sec:setup}. 
% To model the problem, we regard the confounder as a latent variable in representation learning, and propose a VAE-based method that identifies the TEs. 
Instead, many methods assume special structures among the variables, such as instrumental variables (IVs) \citep{angrist1996identification}, proxy variables \citep{tchetgen2020introduction}, network structure \citep{veitch2019using}, and multiple causes \citep{wang2019blessings}. Among them, instrumental and proxy variables are most commonly exploited. \textit{Instrumental variables} are not affected by unobserved confounders, influencing the outcome only through the treatment. On the other hand, \textit{proxy variables} are causally connected to unobserved confounders, but are not confounding the treatment and outcome by themselves. Other methods use restrictive parametric models \citep{allman2009identifiability}, or only give interval estimation \citep{manski2009identification, kallus2019interval}.

% \textcolor{red}{[Approaches to TEs, including classical ones, should be described.]}

In this work, we challenge the problem of estimating TEs under unobserved confounding. We in particular discuss the \textit{individual-level} TE, which measures the TE conditioned on the covariate, for example, on a patient's personal data. 
We highlight the natural VAE architecture following from modeling sufficient scores and the promising experimental results, under unconfounded, IV, proxy, and networked confounding settings. On the theoretical side, we show identification of TEs using our generative model under unconfoundedness, but also discuss a parallel work \citep{wu2021beta} addressing limited overlap and future work(s) under unobserved confounding.

Our method exploits the important concepts of sufficient scores for TE estimation \citep{hansen2008prognostic,rosenbaum1983central} and also the recent advance of VAE with \textit{identifiable} latent variable, which is determined by the true observational distribution \citep[iVAE]{khemakhem2020variational}.
% The hallmark of deep neural networks (NNs) is that they can learn representations of data. 
% It is desirable that the learned representations are %\textit{interpretable}, 
% interpretable, that is, in approximately the same relationship to true latent sources for each down-stream task. 
% A principled approach to interpretable representations is identifiability, that is, when optimizing our learning objective w.r.t.~the representation function, only a unique optimum, which represents the true latent structure, will be returned. Our method provides the stronger identifiability that gives \textit{balanced} representation.
%needed in causal inference. % (maybe up to certain trivial transformation) 
VAEs \citep{kingma2019introduction} are suitable for causal estimation thanks to its probabilistic nature. However, most VAE methods for TEs, e.g., \cite{louizos2017causal,zhang2020treatment}, are ad hoc and thus not identifiable. 
Instead, our goal is to build a VAE that can identify and recover from observational data a sufficient score via the latent variable, which can be seen as a \textit{causal representation} \citep{scholkopf2021toward};
% , which can be used to identify and estimate TEs by \eqref{eq:cate_by_bts} or \eqref{eq:cate_by_bs}. 
recovering the true confounder is not necessary. 
The code is uploaded to OpenReview, and the proofs are in Appendix~\ref{sec:proofs}. Our main contributions are: 
% \vspace{-5pt}
\begin{enumerate}[topsep=0pt, partopsep=0pt, itemsep=0pt, parsep=0pt, leftmargin=13pt]
\item[1)] A new identifiable VAE, Intact-VAE, as a balanced estimator for individualized TEs;
\item[2)] %usage of the prognostic score \citep{hansen2008prognostic} to motivate, and an inspired concept to develop and justify our method;
% Theory, with newly introduced B*-score, of identifiability, identification, and estimation of TEs;
Experimental comparison to state-of-the-art methods under diverse settings;
\item[3)] Proof of TE identification via recovery of sufficient scores, under unconfoundedness;
\item[4)] Discussions of further theoretical developments and principled TE estimation using VAEs.
\end{enumerate}
An early version of this work, which proposed the same VAE architecture, is in \cite{wu2021identifying}.

% The main contributions of this paper are as follows: 
% \vspace*{-2mm}
% \begin{itemize}[topsep=0pt, partopsep=0pt, itemsep=0pt, parsep=0pt, leftmargin=11pt]
% \item Interpretable, causal representation learning by a new VAE architecture for estimating TEs under unobserved confounding; 
% \item Theoretical analysis of the identifiability of representation and TE; 
% \item Experimental study on diverse settings showing performance of  state-of-the-art.
% \end{itemize}
% \begin{itemize}
%     \item Interpretable, causal representation learning by a new VAE architecture for estimating TEs under unobserved confounding;
%     \item Theoretical analysis of the identifiability of representation and TE;
%     \item Experimental study on diverse settings showing performance of  state-of-the-art.
% \end{itemize}

\subsection{Related work}

As detailed in Sec.~\ref{sec:te_vae}, current VAE methods for TE estimation are more heuristic than ``causal''. Our work endeavors to remedy this situation. Below, we focus on other related works.

\textbf{Identifiability of representation learning.} The hallmark of deep neural networks (NNs) is that they can learn representations of data. 
% It is desirable that the learned representations are %\textit{interpretable}, 
% interpretable, that is, in approximately the same relationship to true latent sources for each down-stream task. 
A principled approach to interpretable representations is identifiability, that is, when optimizing our learning objective w.r.t.~the representation function, only a unique optimum, which represents the true latent structure, will be returned. With recent advances in nonlinear ICA, identifiability of representations is proved under a number of settings, e.g., auxiliary task for representation learning \citep{hyvarinen2016unsupervised, hyvarinen2019nonlinear} and VAE \citep{khemakhem2020variational}. The results are exploited in causal discovery \citep{pmlr-v108-wu20b} and causal representation learning \citep{shen2020disentangled}. To the best of our knowledge, this work is the \textit{first} to explore this % new possibility
identifiability in TE estimation.  % causal inference 

\textbf{Causal inference with auxiliary structures.}
CEVAE \citep{louizos2017causal} relies on the strong assumption that the true confounder distribution can be recovered from proxies. Our method is quite different in motivation, applicability, architecture. Detailed comparisons are given in Appendix \ref{sec:cevae}. Also with proxies, \cite{kallus2018causal} use matrix factorization to infer the confounders, and
% , and give consistent ATE estimator. 
\cite{mastouri2021proximal} use kernel methods to solve the underlying Fredholm integral equation.
% and give two consistent ATE estimators.
% \cite{miao2018identifying} establish conditions for identification using more general proxies, but without practical estimation method. 
IVs are also exploited in machine learning, there are methods using deep NNs \citep{hartford2017deep} and kernels \citep{singh2019kernel,muandet2019dual}. 
% Different to the above methods, our method is based on more general concepts of sufficient scores for causal inference, \textit{without} assuming specific auxiliary structures. 
% Also, our method gives consistent estimator for \textit{nonlinear} outcome models, given consistent learning of VAE. 
%This work is influenced by the proxy variable approach. 
% However, CEVAE assumes a specific causal graph where the covariates should be independent of the treatment given the confounder. 

% \paragraph{VAE and Disentangling} Particularly betaVAE [GIN][Causal VAE]

\textbf{Representation learning for causal inference.} Recently, researchers start to design representation learning methods for causal inference, but mostly limited to unconfounded settings. Some methods focus on learning a balanced representation of covariates, e.g., BLR/BNN \citep{johansson2016learning}, and TARnet/CFR \citep{shalit2017estimating}. 
% Adding to this, \cite{yao2018representation} also exploit the local similarity of between data points.
\cite{shi2019adapting} use similar architecture to TARnet, considering the importance of treatment probability. There are also methods using GAN \citep[GANITE]{yoon2018ganite} and Gaussian process \citep{alaa2017bayesian}. Our method shares the idea of balanced representation learning.
% , and further extends to unobserved confounding.

\begingroup

\section{Preliminaries}
\label{sec:setup}

\subsection{Treatment effects and identification}
Following \cite{imbens2015causal}, we begin by defining  \textit{potential outcomes} (or \textit{counterfactual outcomes}) $\rvy(t), t \in \{0,1\}$, which are the outcomes that would have been observed, if we applied treatment value $\rt=t$. Note that, for a unit under research, we can observe only one of $\y(0)$ or $\y(1)$, corresponding to the factual treatment applied. This is the \textit{fundamental problem of causal inference}.  We also observe relevant covariate $\rvx$, which is associated with individuals, and the observation $\data\coloneqq(\rvx,\rvy,\rt)$ is a random variable with underlying probability distribution. 
%, or using do-calculus: $\rvy|do(\rt=t)$

The expected potential outcome is $\mu_t(\vx) = \E(\rvy(t)|\rvx=\vx)$, conditioned on $\rvx=\vx$.
% \begin{equation}
% \displaystyle
%     \mu_t(\vx) = \E(\rvy(t)|\rvx=\vx)% = \E(\rvy|\rvx=\vx, do(\rt=t))
% \end{equation}
The estimands in this work are Average TE (ATE) $\nu$ and Conditional ATE (CATE) $\tau$, defined by 
\vspace{-.05in}
\begin{equation}
\label{eq:ce}
    \tau(\vx) = \mu_1(\vx) - \mu_0(\vx),\quad \nu = \E(\tau(\vx)).
% \vspace{-.1in}
\end{equation}
CATE can be seen as an \textit{individual-level} TE, if conditioned on
%high dimensional and highly diverse 
highly discriminative covariates.

\setlength{\columnsep}{8pt}
\setlength{\intextsep}{5pt}
\begin{wrapfigure}{r}{0.18\textwidth}
  \vspace{-.2in}
  \begin{center}
    \includegraphics[width=0.18\textwidth]{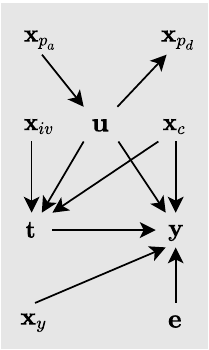}
  \end{center}
  \vspace{-.2in}

  \caption{
%   \newdimen\origiwspc%
%   \newdimen\origiwstr%
%   \origiwspc=\fontdimen4\font% original inter word space
%   \origiwstr=\fontdimen3\font% original inter word stretch
%   \fontdimen1\font=0.1ex
  \footnotesize{A possible causal graph of unobserved confounding.}
%   \fontdimen4\font=\origiwspc
  } % See Appendix for more plots.
  
% \vspace{-.1in}
\label{fig:setting}
\end{wrapfigure}

\textit{Identification of TEs} means that, the true observational distribution uniquely determines and gives the ATE or, better, CATE. 
%In general, we need three conditions for identification of TEs\citep{rubin2005causal}\citep[Ch.~3]{hernanCausalInferenceWhat2020}. 
Adapting standard identification results \citep{rubin2005causal}\citep[Ch.~3]{hernanCausalInferenceWhat2020}, we start with the following conditions, denoted by {\bf (A)}: 
there exists a (possibly unobserved) variable $\rconf \in \R^n$ such that together with $\rvx$, it gives (i) (Exchangeability)
% \footnote{This allows the existence of \textit{observed} confounders in $\rvx$.}
$\rvy(t) \independent \rt| \rconf, \rvx$ and (ii) (Overlap, or Positivity) $p(\rt|\rconf, \rvx) > 0$; and (iii) (Consistency of counterfactuals) $\rvy = \rvy(t)$ if $\rt=t$. All of them are satisfied for \textit{both} $t$, which is our convention when $t$ appears in a statement without quantification. 
% Exchangeability means, just as in RCTs, but additionally given $\x$, that  there is no correlation between factual $\rt$ and potential $\ry(t)$. 
Intuitively, exchangeability means all confounders are in essence contained in $(\rconf,\x)$, and overlap means each possible value of $(\rconf,\x)$ is observed for both treatment groups.
Note that, joint exchangeability $\rvy(0),\rvy(1) \independent \rt| \rvx,\rconf$ is stronger than exchangeability and is not necessary for identification \citep[pp.~15]{hernanCausalInferenceWhat2020}.
% Overlap means that the supports of $p(\vx|t=0)$ and $p(\vx|t=1)$ should be the same, and this ensures that there are data for $\mu_t(\vx)$ on any $(\vx, t)$.
% See Appendix \ref{??} for detailed explanations.

A general example of causal structure that satisfies the three conditions is shown in Figure \ref{fig:setting}, although further structural constraints might be necessary for theoretical developments (see Sec.~\ref{sec:conf}).. 
Here, $\rvx_{c}$,$\rvx_{iv}$,$\rvx_{p_a}$,$\rvx_{p_d}$,$\rvx_{y}$ are covariates that are: (observed) confounder, IV, antecedent proxy (that is antecedent of $\rvz$), descendant proxy, and antecedent of $\rvy$, respectively. The covariate(s) $\rvx$ may \textit{not} have subsets in any categories in the graph. $\rve$ is unobserved exogenous noise on $\rvy$. Assumption {\bf (A)} may hold otherwise, e.g., $\rvx$ is a child of $\rt$. 

%Using the conditions in the second equality below, 
CATE can be given by \eqref{eq:id}, using assumption \textbf{(A)} in the second equality.
% \vspace{-8mm}
\begin{equation}
\label{eq:id}
% \begin{split}
    \mu_t(\vx) = \E(\E(\rvy(t)|\rconf,\vx)) = \E(\E(\rvy|\rconf,\vx,\rt=t)) =\textstyle \int (\int p(y|\conf,\vx,t)ydy)p(\conf|\vx)d\conf.
% \vspace{-8mm}
% \end{split}
\end{equation}

%In this work, we consider \textit{unobserved confounding}, that is, we assume the existence of \textit{confounder(s)} $\rconf$ as above, but it is
% (partially)\footnote{As we will see, since $\rconf$ is the latent variable(s) for VAE and is learned from observations by the VAE, it can contain \textit{all} confounders in principle. Our method will extract the confounding part of $\rvx$ into $\rconf$.} 
If the variable $\rconf$ is observed, then \eqref{eq:id} identifies CATE. However, if $\rconf$ is an \textit{unobserved confounder},
% Still, we have a base on which we motivate and justify our method. 
the naive regression $\E[\rvy| \rvx=\vx, \rvt=t]$ based on observable variables is not equal to $\mu_t(\vx)$. In fact, if an unknown factor correlates with $\rvt$ positively and tends to give higher value for $\rvy$, the naive regression $\E[\rvy| \rvx=\vx, \rvt=1]$ should be higher than $\E[\rvy(1)|\rvx=\vx]$.

\subsection{Prognostic  score}
% Here, we motivate and introduce our method by equation \eqref{eq:id} and Theorem \ref{bscore}, while theoretical justifications are deferred to Sec. \ref{sec:id_bility_model} and \ref{sec:treatment}.

Our method models prognostic scores \citep{hansen2008prognostic}, adapted as \textit{Pt-scores} in this paper,
% to account confounder $\rconf$, 
closely related to the important concept of balancing score \citep{rosenbaum1983central}. Both are sufficient scores for identification; prognostic scores are sufficient statistics of outcome predictors and balancing score is for the treatment (see Appendix \ref{sec:scores} for details).

\begin{definition}
% [P*-scores]
\label{def:scores}
A \textit{Pt-score} (PtS) is two functions $\PS_t(\rconf,\x)$ ($t=0,1$) such that $\rvy(t)\independent\rconf,\x|\PS_t(\rconf,\x)$.
% $\rvy(1)\independent\rvv|\PS_1(\rvv)$. 
%A \textit{PtS} $\PS_{t}$\footnote{Do not confuse PtS with $\PS_{\rt}$ (i.e., we have $\PS_t$ if $\rt=t$). The latter is the \textit{factual} PtS. PtS gives \textit{both} $\PS_t,\PS_{1-t}$, regardless of the factual assignment $\rt=t$, and works under counterfactual assignment.} is both $\PS_0$ and $\PS_1$. 
A PtS is called a \textit{P-score} (PS) if $\PS_0=\PS_1$.
\end{definition}
The identity function is a trivial PS. If the true data generating process (DGP) satisfies additive noise model, i.e., $\y=\vf\st(\rconf,\x,\rt)+\rve$, then $\vf\st_t$ is a PtS \citep{hansen2008prognostic}; and it is a causal representation \citep{scholkopf2021toward} of the direct cause on $\y$, summarizing the effects of $(\rconf,\x)$.
The independence property of PtS (Lemma \ref{prop:pscore} in Appendix),
\begin{equation}
\label{eq:ps_indp}
    \rvy(t)\independent\rt,\rconf,\rvx|\PS_t(\rconf,\rvx),
\end{equation}
is used in second equality of \eqref{eq:cate_by_bts} in Theorem \ref{cate_by_bts} which extends Proposition 5 in \cite{hansen2008prognostic}. 
% We use the same symbol to denote a score and the random variable defined by it, as in \eqref{eq:ps_indp} and \eqref{eq:cate_by_bts}.

\begin{theorem}[CATE by PtS]
\label{cate_by_bts}
If $\PS_t$ is a PtS,
% and $\y|\PS_{\hat t}(\rconf,\x),\rt \sim p_{\y|\PS_{\hat t},\rt}(\vy|P,t)$ where $\hat{t}\in \binset$, 
then CATE can be given by
\begin{equation}
\label{eq:cate_by_bts}
    \begin{split}
        \mu_{t}(\vx) &= \E(\E(\rvy({t})|\PS_{t}(\rconf,\vx),\vx)) = \E(\E(\rvy|\PS_{t}(\rconf,\vx),{t})) \\ &=\textstyle \int (\int p_{\y|\PS_{ t},\rt}(y|P,{t})ydy)p_{\PS_t|\x}(P|\vx)dP, 
        % ,t \in\{0,1\}
    \end{split}
\end{equation}
where $\y|\PS_{t}(\rconf,\x),\rt \sim p_{\y|\PS_{t},\rt}(\vy|P,t)$ and $\PS_{t}(\rconf,\x)|\x\sim p_{\PS_t|\x}(P|\vx)$.
% where $P_{t}\coloneqq\PS_{t}(\conf,\vx)$ is a value of $\PS_{t}$. Further, $p(y|P_{t},{t})=p(y|P,{t})$ where $P$ is a value of $\PS_{\rt}$, since $\rt=t$ is given in the condition.
% \footnote{When necessary, we denote the index on potential outcomes as $t$, and distinguish it from the value of $\rt=t$.}
\end{theorem}
Compared to \eqref{eq:id},
$P=\PS_t(\conf,\vx)$
% $P=\textstyle \int \PS(\conf,\vx)p(\conf|\vx)d\conf$
plays the role of $\conf$, and $p_{\y|\PS_{t},\rt}$ conditions on $P$ instead of $(\conf,\vx)$.
% , as $\rvy\independent \rconf,\rvx|\PS,\rt$ implied by \eqref{eq:ps_indp}. 
% Our VAE introduced in Sec.~\ref{sec:arch} models the outcome distribution $p_{\y|\PS_{t},\rt}$ and the score distribution $p_{\PS_t|\x}$ directly  .
In general, information from $\rconf$ is needed to determine $p_{\y|\PS_{t},\rt}$ and $p_{\PS_t|\x}$. In Sec.~\ref{sec:principled}, we first show that our generative model can identify an equivalent PS if $\rconf$ is observed (Sec.~\ref{sec:unconf}), and then discuss how our model is connected to and might learn relaxations of PtS when $\rconf$ is unobserved (Sec.~\ref{sec:conf}).
% With the knowledge of $\PS_t$, CATE is identified, by choosing one of $\PS_0,\PS_1$ corresponding to the counterfactual outcome of interest. 

% The same reasoning applies to \eqref{eq:cate_by_bs} if we have a P-score $\PS$.

% See Appendix for details on P-score and its relation to balancing score.
% % \vspace{-.1in}
% \begin{equation}
% \label{eq:id_bscore}
% \begin{split}
%     \mu_t(\vx) &= \E(\E(\rvy(t)|\rvb,\vx)) = \E(\E(\rvy|\rvb,\vx,\rt=t)) \\ &=\textstyle \int (\int p(y|\vb,\vx,t)ydy)p(\vb|\vx)d\conf,
% % \vspace{-.5in}
% \end{split}
% \end{equation}
% where $\rvb\coloneqq\bm\beta(\rvz, \rvx)$.
% By ``causal'', we mean the representation can be used to identify or estimate TEs.

% We will build a VAE that can learn to recover $\bm\beta(\rvz, \rvx)$. Recovering the true confounder $\rvz$ would be great, but this is not required, as shown in Theorem \ref{bscore}, which says, for identification of TEs, we only need to have $\bm\beta(\rvz, \rvx)$, a causal representation, which contains the information of the propensity score, that is relevant to treatment assignment. 

\section{Intact-VAE}
\label{sec:intact}
%  and its identifiability
% In this section, we introduce B*-scores motivated by prognostic scores (Sec.~3.1), and our VAE model and architecture, based on the distribution $p(\rvy,\rvz|\rvx,\rt)$ (Sec.~\ref{sec:arch})
% % , and prove identifiability of our model similar to that of iVAE (Sec.~\ref{sec:id_bility_model}). 
% The developments give many hints to causal inference, and enable us to address identification and estimation of TEs in Sec.~\ref{sec:treatment}. 
In this section, we first introduce our generative model and VAE architecture, then prove the identifiability of the model, and finally use our model to estimate TEs. Our method is expected to learn a latent representation sufficient for TE identification/estimation, as we will see in the next section.

\subsection{Model and architecture}
\label{sec:arch}

\begingroup

We build the generative model \eqref{model_indep} of our VAE with the direct cause $\PS_t(\rconf,\x)$ as the latent variable $\z$ and connect the model to iVAE. The outcome distribution $p_{\y|\PS_{t},\rt}$ in \eqref{eq:cate_by_bts} is modeled by $p_{\vf}(\vy|\vz,t)$, and the score distribution $p_{\PS_{\rt}|\x}$ in \eqref{eq:cate_by_bts} is modeled by $p_{\blambda}(\vz|\vx,t)$. The condition on $\x$ in the joint model $p(\rvy,\rvz|\rvx,\rt)$ reflects that our estimand is CATE given $\x$.
% $\vaegen$ as our model, we must have
% \vspace{-.1in}
% \begin{equation}
% \label{model_indep}
%     p(\rvy,\rvz|\rvx,\rt) = p(\rvy|\rvz,\rt)p(\rvz|\rvx,\rt).
% % \vspace{-.1in}
% \end{equation}
\begin{equation}
\label{model_indep}
    \begin{gathered}
        p_{\vtheta}(\vy,\vz|\vx,t) = p_{\vf}(\vy|\vz,t)p_{\blambda}(\vz|\vx,t), \\
        % \quad
        p_{\vf}(\vy|\vz,t)=p_{\beps}(\vy-\vf_{t}(\vz)),
        %\y=\vf_{\rt}(\z)+\beps,
        % , 
        % p_{\vf,\vg}
         \quad 
        p_{\blambda}(\vz|\vx,t) \sim \mathcal{N} ( \vz ; \vh_{t}(\vx),\diag(\vk_{t}(\vx))),
    \end{gathered}
\end{equation}
The outcome assumes an additive noise model such that 
$\beps \sim p_{\beps}$ denotes the exogenous noise; and $p_{\blambda}$ is a factorized Gaussian, where $\blambda_{\rt}(\x)\coloneqq \diag\inv(\vk_{\rt}(\rvx))(\vh_{\rt}(\x), -\frac{1}{2})^T$ is the natural parameter as in the exponential family. $\vtheta\coloneqq(\vf,\blambda)=(\vf, \vh, \vk)$ contains the functional parameters. 
% Modeling the score by a conditional distribution rather than a deterministic function is more flexible.
% , as we will argue in Sec.~\ref{sec:regular}. 

The model is learned by the evidence lower bound (ELBO) which estimates the variational lower bound (See Appendix \ref{vaes} for the basics of VAEs):
% , and we will see its importance later. 
% \vspace{-.1in}
\begin{equation}
\label{elbo}
\begin{split}
    \log \trueobs \geq \log \trueobs - \KL(\encoder\Vert \vaepost) \\ 
    = \E_{\vz \sim q}\log \decoderparam - \KL(\encoder\Vert \condpriorparam) .
    % \coloneqq \mathcal{L}_{Intact-VAE}(\rvx,\rvy,\rt).
\end{split}
% \vspace{-.1in}
\end{equation}
% As in iVAE, the decoder drops the dependence on $\rvx$. 
with $p_{\vf}$ as the decoder, $q$ as the encoder, and $p_{\blambda}$ as the \textit{conditional} prior.

\setlength{\columnsep}{8pt}
\setlength{\intextsep}{5pt}
\begin{wrapfigure}{r}{0.3\textwidth}
  \vspace{-.2in}
  \begin{center}
    \includegraphics[width=0.3\textwidth]{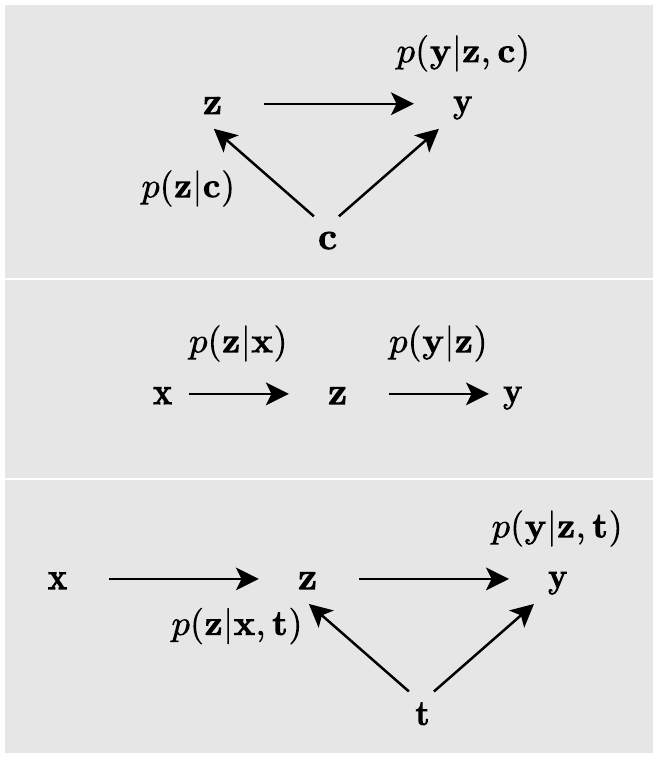}
  \end{center}
  \vspace{-.2in}
  \caption{Decoders of CVAE, iVAE, and Intact-VAE.} % See Appendix for more plots.
\vspace{-.1in}
\label{f:vae}
\end{wrapfigure}

% [**move to estimation?] 
We name this architecture \textit{Intact-VAE} (\textit{I}de\textit{n}tifiable \textit{t}re\textit{a}tment-\textit{c}ondi\textit{t}ional VAE). 
% We naturally have an identifiable conditional VAE (CVAE), as the name suggests. 
Note that \eqref{model_indep} has a similar factorization with the generative model of iVAE, $p_{\btheta}(\vy,\vz|\vx) = p_{\vf}(\vy|\vz)p_{\blambda}(\vz|\vx)$; the first factor (decoder) does not condition on $\x$. Similarly, our decoder $p_{\vf}(\vy|\vz,t)$ conditions on $\z$ as a PtS which satisfies $\rvy\independent \rvx|\z,\rt$.
On the other hand, the condition on $\rt$ is in both the factors of \eqref{model_indep}. Thus, our VAE architecture can be seen as a combination of iVAE and conditional VAE (CVAE) \citep{sohn2015learning,kingma2014semi}, with $\rt$ as the conditioning variable. 
See Figure \ref{f:vae} for the comparison in terms of graphical models. Our encoder $q$, which conditions on all the observables and builds the approximate posterior, is similar to other VAEs. See Appendix \ref{vaes} for summary of CVAE and iVAE.

% Figure \ref{f:vae} depicts the relationship of CVAE, iVAE, and Intact-VAE. Do \textit{not} confuse the graphical model of our generative model with a causal graph. Particularly, do not confuse confounder $\rconf$ with latent variable $\z$ that corresponds to a score. For example, in Figure \ref{fig:setting}, the true, causal generating process we want to address,  confounder $\rconf$ should be the cause of $\rt$, and the causal arrow between $\rconf$ and $\rvx$ could also be reversed. % And this is also why the two deviations from \eqref{eq:id}, introduced in \eqref{model_indep}, do not restrict our method.

% \paragraph{Explain each part of Intact-VAE} []Can be omitted?]

% [We did not add $\rt$ in the conditional prior before. But as presented here, it is in some sense more natural, and I can indeed think of some settings where this would be better. More experiment will done.]

% As to the \textit{latent} representation, we will show that Intact-VAE can learn to recover a causal representation of confounder $\rvz$ and covariate $\rvx$. 

% Readers may notice that, setting $p(\rvz|\rvx,\rt)=p(\rvz|\rvx)$ in \eqref{model_indep}
% % dropping the dependence on $\rt$ in conditional prior $p(\rvz|\rvx,\rt)$ 
% corresponds to \eqref{eq:cate_by_bs} (and also \eqref{eq:id} if we look only at $p(\rconf|\x)$). Indeed, later in Sec.~\ref{sec:estimation}, we will see that, given existence of a B/P-score, this could enable us to recover a B-score as representation, and give a \textit{balanced} estimator.

% We detail the parameterization of Intact-VAE. 
Though our theory allows more general distributions, for tractable inference and easy implementation, we have used factorized Gaussian for the prior, and now also use it for the decoder $p_{\vf,\vg}(\vy|\vz,t)$ and encoder $\encoderparam$; i.e., they are products of univariate Gaussian distributions:
% And this is not restrictive if the mean and variance are given by arbitrary nonlinear functions. 
% \vspace{-.1in}
\begin{equation} 
\label{more_param}
\textstyle
    \rvy|\rvz,\rt \sim \prod_{j}^d \mathcal{N}(y_j;f_j, g_j), %(\vz,t) % \varepsilon_\rt(\rvy-f_\rt(\rvz)), 
    % p_{\vf,\vg}
    \quad 
    \rvz|\rvx,\rvy,\rt \sim \prod_{i=1}^n \mathcal{N}(z_i ; r_i , s_i).
\end{equation}
$\bm\phi\coloneqq(\vr,\vs)$, like $\btheta$, contains functional parameters given by NNs which take $(\x,\y,\rt)$ as inputs.
% (e.g. $\vh\coloneqq(h_i(\rvx,\rt))^T$).
% The \textit{identifiability} of our model, that is, \textit{parameters} $(\vf,\bm\lambda)$ can be identified (learned) up to affine transformation, is not needed until Sec.~\ref{sec:estimation} and thus saved to Theorem \ref{idmodel} in Appendix.
% In this subsection, we gain insight on how to make Intact-VAE learn a B-score. 
% Throughout the paper, we will use subscript to denote the version of a quantity when $\rt=t$ is given, e.g., $\vh_t(\rvx)\coloneqq(h_i(\rvx,t))^T$
We often write $t$ of the function argument in subscripts.

\subsection{Identifiability of model parameters}

Our model identifiability lays the foundation for principled causal effect estimation, as we will discuss in Sec.~\ref{sec:principled}. The main theoretical result of iVAE is extended in our Theorem \ref{idmodel} which combines the techniques from \cite{khemakhem2020variational} and \cite{sorrenson2019disentanglement}. Essentially the same results can be proved for other exponential family priors. 

\endgroup

\begin{theorem}[Model identifiability]
\label{idmodel}
Given $p_{\vtheta}(\vy,\vz|\vx,t)$ specified by \eqref{model_indep}
% \footnote{We specified factorized Gaussians in \eqref{model_param} and they show good performance in our experiments. But our theorems can be extended to general exponential families, see \cite{khemakhem2020variational}.}
, for $\rt=t$, assume  

\renewcommand{\labelenumi}{\roman{enumi})}
\begin{enumerate}
\def\theenumi{\roman{enumi})}

\item $\vf_t(\rvz)$ 
% \footnote{Here we mean $\vf_t(\vz)$. In \textit{this subsection and related Sec.~\ref{notenough}}, we will refer to quantities when $\rt=t$ is given, and we will omit the subscripts $t$ when appropriate.} 
is injective and differentiable;

% \item $\vg(\rvz)=\bm\sigma_{\rvy,t}$ is constant (i.e. $g_i(\rvz,t)=\sigma_{\rvy_i,t}$);
% \item $\vg_t$ is fixed (i.e. $\vg_t$ is in fact \emph{not} a parameter);

% \item \label{ass:inv_jac} there exists $\vx^o$ where $\mJ_{\bm\lambda}(\rvx)$, the Jacobian of $\bm\lambda(\rvx)\coloneqq(\vh(\rvx,t),\vk(\rvx,t))^T$, is invertible.
\item \label{ass:inv_jac} there exist $2n+1$ points $\vx_0,...,\vx_{2n}$ such that the $2n$-square matrix $\mL_t \coloneqq [\bgamma_{t,1},...,\bgamma_{t,2n}]$ is invertible, where $\bgamma_{t,k}\coloneqq\blambda_t(\vx_k)-\blambda_t(\vx_0)$.
\end{enumerate}

Then, given $\rt=t$, the family is \emph{identifiable} up to an equivalence class. That is, if $p_{\vtheta}(\rvy|\rvx,\rt=t)=p_{\vtheta'}(\rvy|\rvx,\rt=t)$\footnote{$\vtheta'=(\vf',\vh',\vk')$ is another parameter giving the same distribution. In this paper, symbol $'$ (prime) always indicates another parameter (variable, etc.) in the equivalence class.
% , which will be learned by VAE. In this paper, symbol ``$'$'' (prime) always indicates parameters (variables, etc.) learned/recovered by VAE.
}, we have the relation between parameters: for any $\vy_t$ in the image of $\vf_t$,
% \vspace{-.1in}
\begin{equation}
\label{eq:class}
     \vf_t^{-1}(\vy_t) = \diag(\va){\vf'_t}\inv(\vy_t) + \vb %\coloneqq 
     =:\mathcal{A}_t({\vf'_t}\inv(\vy_t))%\exists \mA, \bbeta: \forall \vy \in \mathcal{R}_\rvy_t:  $\mathcal{R}_\rvy_t$ is the support of $\rvy_t$,
% \vspace{-.1in}
\end{equation}
%, and noise $\bm\epsilon$ and
where $\diag(\va)$ is an invertible $n$-diagonal matrix and $\vb$ is a $n$-vector, both depend on $\bm\lambda_t$ and $\bm\lambda'_t$. 
\end{theorem}
The essence of Theorem \ref{idmodel} is $\vf'_t = \vf_t \circ \mathcal{A}_t$, that is, $\vf_t$ can be identified (learned) up to an affine transformation defined by $\blambda_t$. This identifiability of parameters does not directly imply TE identification; other assumptions are needed, as we show in Sec.~\ref{sec:principled}.
The assumptions in Theorem \ref{idmodel} are inherited from iVAE. 
% Also, by definition of inverse, we have $\vf' = \vf \circ \mathcal{A}$, and this is the essence of the identifiability. %[better use $\vc$] % \textcolor{red}{on uniqueness. Formal condition?} \mA\vf'^{-1}(\rvy) + \vb
Injectivity in i) is related to TE identification (Theorem \ref{th:id_os}), which requires injectivity in the true DGP and identifies a PtS up to an injective mapping. Intuitively, if ii) does not hold, then the support of $\blambda(\x)$ degenerates to a $(2n-1)$-dimensional space; thus, ii) holds easily in practice (see B.2.3 in \cite{khemakhem2020variational}). Note that, to have \eqref{eq:class}, we only need the same observational distribution $p(\rvy|\rvx,\rt=t)$, but this leaves room for different latent distributions.

\subsection{Causal effect estimation}
Once the VAE is learned by the ELBO \eqref{elbo}, 
the estimate of $\mu_t$ is given by 
\begin{equation}
\label{eq:cate_est}
    % \hat{\mu}_{\hat{t}}(\vx)=\E_{q_{\hat{t}}^{ad}(\z|\x=\vx)}\vf_{\hat{t}}(\z)=\E_{\data|\vx \sim p(\y|\vx)}\E_{\z}\vf_{\hat{t}}(\z)q_{\bphi}(\z|\vx,\y,\rt=\hat{t}), {\hat{t}} \in \binset
    \hat{\mu}_{\hat{t}}(\vx)=\E_{q(\vz|\vx)}\vf_{\hat{t}}(\vz)=\E_{\data|\vx \sim p(\vy,t|\vx)}\E_{\vz\sim q_{\bphi}}\vf_{\hat{t}}(\vz), {\hat{t}} \in \binset,
\end{equation}
% $q_t^{ad}(\z|\x)\coloneqq \E_{p(\y|\vx)}q_{\bphi}(\z|\vx,\y,\rt=t)$ 
where $q(\vz|\vx)\coloneqq \E_{p(\vy,t|\vx)}q_{\bphi}(\vz|\vx,\vy,t)$
is the aggregated posterior.
Our overall \textbf{algorithm} steps should be clear. After training Intact-VAE by \eqref{elbo}, we feed $(\vx,\vy,t)$
% $\mathcal{D}=\{(\vx, y)\coloneqq (\vx, y, t)\}$ 
into the encoder and draw posterior sample from it. 
% =\delta(\rvz' - \vr'(\vx,y))
Then, setting $\rt=\hat{t}$ in the decoder, feed the posterior sample into it, 
% $\{\vz' = \vr'(\vx,y)\}$, 
we get counterfactual sample as outputs of the decoder.
% $p(\rvy'|\rvz'=\vz', \rt=\hat{t})=\delta(\rvy'_{\hat{t}}-\vf'_{\hat{t}}(\vz'))$.
Finally, we infer CATE $\hat{\tau}(\vx)=\hat{\mu}_1(\vx)-\hat{\mu}_0(\vx)$.
% [counterfactual assignment better stressed earlier]

% As mentioned in Introduction, a whole line of work aims to design better, balanced, estimator with observed confounding. 
% Recall that the main problem of naive regression (e.g., \eqref{eq:id} would be naive if $\rconf$ was observed) is imbalance. 
Our method also addresses the problem of imbalance--if $p(\x|0),p(\x|1)$ are very different for some $\vx$, then we have few data points for one of $\mathcal{D}|\vx,t\coloneqq\{(\vx,\y,t)\}$, resulting in poor estimation. 
% The estimator \eqref{estimator} also addresses imbalance to some extent, by learning a representation that is \textit{lower} dimensional than $\x$ (see also \citep{d2020overlap}), as Bt-scores often are. 
We set $\blambda_0=\blambda_1$ in the prior, 
% and thus $\mathcal{A}_0=\mathcal{A}_1$ in \eqref{eq:class}. 
% Assumption \ref{ass:bcovar} is the key, and ensures learning B-scores, not only Bt-scores. It introduces stronger model identifiability than iVAE. i) is a technical assumption inherited from iVAE\footnote{We also omit technical assumption iii) of Theorem \ref{idmodel}, which is not very relevant, and our general balancing assumption implies it.}. 
% is stronger than iii) in Theorem \ref{idmodel}
% \ref{ass:bcovar} adds balancing into our estimator. 
and thus $p_{\bm\lambda}(\rvz|\rvx,\rt)=p_{\bm\lambda}(\rvz|\rvx)$ is independent of $\rt$ given $\rvx$. 
% Just like balancing score gives balanced estimator, with B-score, we have correspondence between \eqref{bestimator} and \eqref{eq:cate_by_bs} where $p(\BS|\x)$ is the score distribution, in addition to that between \eqref{estimator} and \eqref{eq:cate_by_bts}.
% set $\bm\lambda=\bm\lambda(\rvx)$, which does \textit{not} depend on $\rt$, in our model \eqref{model_param}. Then \ref{ass:bcovar} is trivially satisfied since $\bm\lambda_0=\bm\lambda_1=\bm\lambda(\rvx)$ with $\mC=\mI,\vd=\bm0$. Note this means that the conditional prior $p_{\bm\lambda}(\rvz|\rvx,\rt)=p_{\bm\lambda}(\rvz|\rvx)$ is independent of $\rt$ given $\rvx$. 
% However, this is very different to $p^*(\rvz|\rvx,\rt)=p^*(\rvz|\rvx)$, which might make the problem trivial (See Appendix). Again, recall that we only need to learn observational distribution as in assumption \ref{ass:opt}, and this is still possible if $\bm\lambda(\rvx)$ is flexible. The \textit{latent} representation can vary freely, we might have $p^*(\rvz|\rvx,\rt) \neq p_{\bm\lambda}(\rvz|\rvx,\rt)=p_{\bm\lambda}(\rvz|\rvx)$, $p^*(\rvz|\rvx,\rt) \neq p^*(\rvz|\rvx)$, and even $p^*(\rvz|\rvx) \neq p_{\bm\lambda}(\rvz|\rvx)$. We will see this is favorable in experiments. 
The same prior for the treatment groups, i.e., the balanced prior of latent representation, is similar to balanced representation learning \citep{johansson2016learning, shalit2017estimating}, where balanced representation is favored by ad hoc regularization. This is also related to the fact that, when building CVAE, unconditional prior can achieve better performance \citep{kingma2014semi}.

% In both equations, we can instead use $\vz$ from the conditional prior, but, of course, the posterior sample has better finite sample performance.
By taking $\y$, the posterior model (the encoder) is better than the prior.
On the other hand, sampling posterior requires \textit{post-treatment} observation $\vy$. Often, it is desirable that we can also have \textit{pre-treatment} prediction for a new subject, with only the observation of its covariate $\rvx=\vx$. To this end, in \eqref{eq:cate_est}, we use conditional prior $p(\rvz|\rvx)$ as a pre-treatment predictor for $\rvz$: draw sample from $p(\rvz|\vx)$ instead of $q$ and get rid of the outer average taken on $\data$; all the others remain the same. 
% We will also have sensible pre-treatment estimation of TEs, as ensured by Theorem \ref{id_tion}.

As usual, we expect the variational inference and optimization procedure to be (near) optimal; that is, the consistency of VAE. Consistent estimation using the prior is a direct corollary of TE identification and the consistent VAE. See Appendix \ref{sec:proofs} for formal statements and proofs. It is possible to prove the consistency of the posterior estimation, as shown in \cite{bonhomme2021posterior,liao2011posterior}, and we leave this for future work, see Sec.~\ref{sec:conf}.

\section{Principled causal effect estimation via VAEs: retrospect and prospect}
\label{sec:principled}
In this section, we present a critical review of existing VAE-based methods in Sec.~\ref{sec:te_vae} and then discuss theoretical developments that can lead to principled causal estimation, under different settings. 
% , particularly the independence properties of PGS.  
% As a consequence, our VAE architecture is a combination of iVAE and CVAE. 
% Moreover, in our Intact-VAE, pre-treatment prediction is given naturally by our conditional prior.
% , thanks to the correspondence between our model and \eqref{eq:cate_by_bts}.

% Most VAE methods for causal effects, e.g. [cites], add ad hoc heuristics into their VAEs, and thus break down probabilistic modeling, not to mention identifiable representation (see Appendix for details). 

\subsection{VAEs for TE estimation}
\label{sec:te_vae}

Most VAE methods for TEs, e.g., \cite{louizos2017causal,zhang2020treatment,vowels2020targeted,lu2020reconsidering}, add ad hoc heuristics into the VAEs, and thus break down probabilistic modeling, not to mention model identifiability. Moreover, the methods learn representations from proxy variables, leading to either impractical assumptions or conceptual inconsistency, in TE identification. 

\textbf{On identification.} First, as to TE identification, CEVAE assumes unobserved confounder can be recovered, which is rarely possible even under further structural assumptions \citep{tchetgen2020introduction}. Indeed, \cite{rissanen2021critical} recently give evidence that the method often fails. Other methods \citep{,zhang2020treatment,vowels2020targeted,lu2020reconsidering} assume unconfoundedness but still rely on proxy at least intuitively; for example, \cite{lu2020reconsidering} factorize the decoder as if in the proxy setting. However, \textit{unconfoundedness and proxy should not be put together}. The conceptual inconsistency is that, by definition, unconfoundedness means covariates \textit{fully} control confounding, while the motivation for proxy is that unconfoundedness is often \textit{not} satisfied in practice and covariates are at best proxies of confounding, which are non-confounders causally connected to confounders \citep{tchetgen2020introduction}. 
Second, without model identifiability, the empirical results of the methods lack solid ground; under settings not covered by their experiments, the methods would silently fail to learn proper representations, as we show in Sec.~\ref{sec:exp_syn}.

\textbf{On ad hoc heuristics.} Ad hoc heuristics break down probabilistic modeling and/or give ELBOs that do not optimally estimate the models. For example, in CEVAE, $q(\rt|\x)$ and $q(\y|\x,\rt)$ are added into the encoder to have pre-treatment estimation, and the ELBO has two additional likelihood terms respectively. The VAE in \cite{zhang2020treatment} is even more ad hoc; it splits the latent variable $\z$ into three components, and applies the ad hoc tricks of CEVAE to each of the component. Particularly, when constructing the encoder, they implicitly assume the three components of $\z$ are conditional independent give $\x$, which violates the intended graphical model. 

Compared to the above methods, our Intact-VAE is simpler and more principled, and often has better performance. It models a prognostic score as the latent variable and is based on the identification equation \eqref{eq:cate_by_bts}, while not compromised by ad hoc heuristics. Our ELBO is derived by standard variational lower bound \eqref{elbo}. Moreover, our pre-treatment prediction is given naturally by the prior, thanks to the correspondence between our model and \eqref{eq:cate_by_bts}. 
We show in the following subsections how our model and its identifiability inspire theoretical developments in TE identification.

\subsection{Identification under unconfoundedness}
\label{sec:unconf}

As a first step, we take up theoretical analysis of Intact-VAE under assumption \textbf{(A)} when $\rconf$ is observed (i.e., $\rconf$ is components of $\x$). Then, we have $\y(t) \independent \rt|\x$ and $p(\rt|\x)>0$ (Exchangeability and Overlap given $\x$). This is the standard unconfoundedness setting. Regarding to the PtS, $p_{\PS_{\rt}|\x}$ in \eqref{eq:cate_by_bts} degenerates to $\PS_{\rt}(\x)$.
Theorem \ref{th:id_os} is an identification under shape restriction \citep{chetverikov2018econometrics}, because injectivity in assumption i) is monotonicity if $\vj_t$ is on $\R$. 

\begin{theorem}[Identification via PS]
\label{th:id_os}
Use model \eqref{model_indep} under unconfoundedness, further assume
\renewcommand{\labelenumi}{\roman{enumi})}
\begin{enumerate}
\def\theenumi{\roman{enumi})}
\item (Injective separation) $\mu_{t}(\x)=\vj_t(\OS(\x))$
for some function $\OS$ and injective function $\vj_t$;
\item (Score matching) in the model, $n=d$, $\vf_t$ is injective, $\vh_0(\x)=\vh_1(\x)$, and $\vk(\x) = \bm0$.
\end{enumerate}
Then, if $\E_{p_{\btheta}}(\y|\x,\rt)=\E(\y|\x,\rt)$, we have
\renewcommand{\labelenumi}{\arabic{enumi})}
\begin{enumerate}
\def\theenumi{\arabic{enumi})}
\item (Recovery of score) $\z_{\blambda,t}=\vh_t(\x)=\vv(\OS(\x))$ where $\vv$ is an injective function;
% $\vv: \mathcal{O} \to \R^n$ is an injective function, and $\mathcal{O}$ is the support of $\OS(\x)$.
\item (CATE Identification)
$\mu_{{t}}(\x) = \hat{\mu}_{{t}}(\x) \coloneqq \E_{p_{\blambda}(\z| \x,t)} \E_{p_{\vf}}(\y|\z,{t}) = \vf_t(\vh_t(\x))$.
% , for any ${t} \in \binset$ and $\vx \in \mathcal{X}$.
\end{enumerate}
\end{theorem}
In essence, $\mu_t$ is identified up to an invertible mapping $\vv$, such that $\vf_t=\vj_t\circ\vv\inv$ and $\vh_t=\vv\circ\OS$, with \textit{same} $\vv$ for both $t$. 
The connection to PS\footnote{Note that, some call $\mu_0$ the prognostic score (e.g, \cite{schuler2020increasing,tarr2021estimating}), even without additive noise models. In this alternative terminology, we can also say $\OS$ is a PS, without requiring additive noise models.} is clear; if additive noise model is the ground truth, then $\OS$ is a PS because $\mu_t$ is is a PtS. Then, because $\OS$ is recovered up to $\vv$, the independence $\y(t)\independent\x|\OS(\x)$ is preserved. Assumption i) says that the treatment affects $\E(\y|\vx)$ only through injective $\vj_t$ (which is identified up to $\vv$), and CATEs
are given by $\mu_0$ and an invertible function $\vi\coloneqq\vj_1\circ\vj_0\inv$. See Appendix \ref{sec:p_inj} for real-world examples satisfying i) and the connection to ``independent causal mechanisms'' \citep{janzing2010causal}. With the existence of $\OS$, Assumption ii) simply matches the model to the truth. Note that, with $\vk(\x) = \bm0$, the prior $\z_{\blambda,t} \sim p_{\blambda}(\vz|\vx,t)$ degenerates to function $\vh_{t}(\x)$. 

Interestingly, Theorem \ref{th:id_os} does not depend on Theorem \ref{idmodel}, mainly because i) is strong in the sense that it indicates a PS. To work under general additive noise models with only PtSs, we use our model identifiability \eqref{eq:class} in \cite{wu2021beta}, the main ideas of which is briefly summarized below.

\subsection{Identification without overlap of $\x$}
The advantage of PtS can be more clearly seen when $\x$ does not satisfy overlap. Now, a straightforward estimation of TEs is not possible at a non-overlapping value $\vx$ due to lack of data. However, $\PS_t(\x)$ can map some non-overlapping values to an overlapping value, and overlap of $\PS_t(\x)$ implies but is not implied by overlap of $\x$ \citep{d2020overlap}. In a parallel submission \citep{wu2021beta}, we assume the overlap of a PtS instead and extend Theorem \ref{th:id_os} to this \textit{limited overlap} setting. 
% The theoretical analysis is done as a part of a full conference paper under review.
This is a natural next step, because we already model a PtS as the latent variable.

A main result of \cite{wu2021beta} is that the latent variable of our VAE recovers a PtS and identifies the CATE through the model, under limited overlap, similar to the conclusion 1) and 2) in Theorem \ref{th:id_os}. To recover the PtS, we derive a condition and strengthen \eqref{eq:class} so that $\mathcal{A}_0=\mathcal{A}_1$, which compensates for $\PS_0\neq\PS_1$.
Thus, we fully exploit the probabilistic nature of VAE in modeling, and give principled causal estimation based on the identifiability of VAE.

Still, with $\rconf$ observed, it would seem unnecessary to model $\PS_{\rt}(\x)$ by the distribution $p_{\blambda}(\vz|\vx,t)$ (the prior). However, the prior, together with the encoder, quantifies the \textit{uncertainty of scores}--we are uncertain how likely a PS (not only a PtS) can be recovered \citep[Sec.~4.1 \& C.5]{wu2021beta}.
% caused by the hidden noise $\rve$ on $\y$. 

\subsection{Preliminary thoughts under unobserved confounding}
\label{sec:conf}
The positive experimental results motivate us to consider the theory under unobserved confounding.
Moreover, the prior in \eqref{model_indep} is even more natural with $\rconf$ unobserved, since $p_{\PS_{\rt}|\x}$ is not degenerated due to the uncertainty of $\rconf$. Thus, we conjecture that, in our VAE framework, unobserved confounding is treated as a source of uncertainty of scores and is handled in a Bayesian way. We give more considerations for future theoretical work below.

% We note that, under unobserved confounding, recovering PtS is much harder if not impossible. Since $\PS_t$ is now a function of $\x$ and $\rconf$, it is unclear how $\PS_t$ could be learned without $\PS_t \independent \rconf | \x$, which basically means $\x$ blocks paths between $\rconf$ and $\PS_t$ and $\x$ gives exchangeability, in contrary to our purpose. 

\textbf{Identification.\space} Auxiliary structures (e.g., IVs) can give TE identifications via \textit{control functions} $\CF(\rt,\x)$, conditioning on which the treatment becomes exogenous, that is, $\y(t) \independent \rt|\CF(t,\x)$ \citep{matzkin2007nonparametric,wooldridge2015control}. Control functions can be stochastic, as in \cite{puli2020general}. Consistent TE estimation can be given by a regression of outcome on the treatment and a control function. Our model \eqref{model_indep} can be seen as a two-stage procedure: first, $p_{\blambda}(\vz|\vx,t)$ gives a stochastic control function; second, $p_{\vf}(\vy|\vz,t)$ regresses the outcome. We need to specify the control function learned by Intact-VAE and the required structural constraints. Control functions are recently found under the proxy setting \citep{nagasawa2018identification}, or in the presence of both proxies and IVs \citep{tien2021instrumental}.

\textbf{Estimation.\space} In causal inference, many models, including nonparametric IV regression (NPIV), are stated as \textit{conditional moment restrictions} (CMRs) \citep{newey199316}. Optimizing the ELBO of our VAE, given by \eqref{elbo}, can be seen as finding functions $\vf$ and $\CF$, subject to the CMR $\E_{p_{\btheta}}(\y|\x,\rt)=\E(\y|\x,\rt)$. 
We believe our Intact-VAE framework, possibly with modifications, can be shown to give optimal estimation under the CMR.
There are formal connections between CMRs and \textit{quasi-Bayesian} analysis using KL divergence \citep{zhang2006eps,jiang2008gibbs,kim2002limited}.
% in statistics and econometrics. 
% Some works show optimality of quasi-Bayesian under CMR.
For example, \cite{kato2013quasi} uses a quasi-likelihood from the CMR of NPIV to set the prior, and the Gibbs posterior \citep{zhang2006eps,jiang2008gibbs} is a minimizer of an information complexity which has a variational characterization similar to an ELBO. For general CMR models,
% (e.g., NPIV and quantile IV \citep{chernozhukov2013quantile}), 
\cite{liao2011posterior} extend \cite{kim2002limited} and give the best approximation to the true likelihood function under the CMR by minimizing a KL divergence. Very recently, \cite{wang2021scalable} employ quasi-Bayesian analysis to kernel-based IV methods, but only consider unconditional moments.

\section{Experiments}
%As mentioned, \textit{all} the true data generating processes in the experiments, are special cases of injective outcome model in Lemma 1 (in fact, of Theorem \ref{id_tion} that is already a special case). Thus, \textit{identification is ensured} by Theorem \ref{id_tion}. Generality of our results is again seen from the diverse experiment settings and the state-of-the-art performance.

We use the proposed Intact-VAE for three types of data, and compare it with existing methods.

%in all the experiments, learned noise at least matches fixed one, and sometimes significantly better. Also refer to the end of Sec.~\ref{sec:identification} for rationale.
Unless otherwise indicated, for each function $\vf,\vg,\vh,\vk,\vr,\vs$ in our VAE, we use a multilayer perceptron (MLP) that has 3*200 hidden units with ReLU activation, and $\bm\lambda=(\vh,\vk)$ depends only on $\rvx$. 
% Note that, while Theorem~\ref{id_tion} assumes the outcome noise $\vg$ is fixed and known, we train $\vg$ also.  
The Adam optimizer with initial learning rate $10^{-4}$ and batch size 100 is employed. 
All experiments use early-stopping of training by evaluating the ELBO on a validation set. We test post-treatment results on training and validation set jointly. 
% This is non-trivial. Recall the fundamental problem of causal inference in Introduction and Sec.~\ref{sec:setup}.
The treatment and (factual) outcome should not be observed for pre-treatment predictions, so we report them on a testing set.
More details on hyper-parameters and settings are given in each experiment and Appendix.

As in previous works \citep{shalit2017estimating, louizos2017causal}, we report the absolute error of ATE $\epsilon_{ate}\coloneqq |\E_{\mathcal{D}}(y(1)-y(0)) - \E_{\mathcal{D}}\hat{\tau}(\vx)|$, and the square root of empirical PEHE \citep{hill2011bayesian} $\epsilon_{pehe}\coloneqq \E_{\mathcal{D}}((y(1)-y(0))-\hat{\tau}(\vx))^2$ for individual-level TEs.

\subsection{Synthetic dataset}
\label{sec:exp_syn}

\begin{figure*}[h]
    \centering
    \includegraphics[width=.9\textwidth]{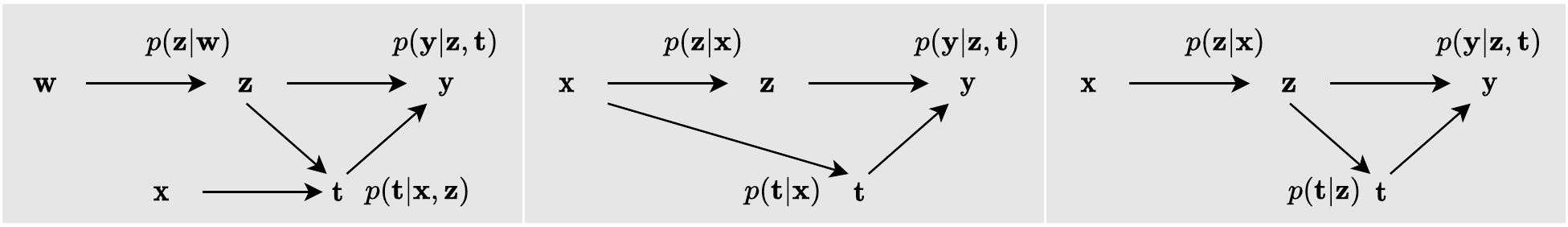}
    \vspace{-.1in}
    \caption{\footnotesize{Graphical models for generating synthetic datasets. From left: IV $\rvx$, ignorability given $\rvx$, and proxy $\rvx$. Note that in the latter two cases, reversing the arrow between $\rvx, \rvz$ does not change any independence relationships, and causal interpretations of the graphs remain the same.}}
    \label{fig:art}
\end{figure*}
\vspace{-.2in}
\begin{equation}
\label{art_model}
\begin{split}
\textstyle
    \rvx \sim \mathcal{N}(\bm\mu, \bm\sigma); 
    \rz|\rvx \sim \mathcal{N}(h(\rvx), \beta k(\rvx)); 
    \rt|\rvx,\rz \sim \bern(\logi(l(\rvx,\rz))); 
    \rvy|\rz,\rt \sim \mathcal{N}(f(\rz,\rt), \alpha). %, \mu_i\sim \uniform(-0.2, 0.2), \sigma_i \sim \uniform(0, 0.2)
\end{split}
\end{equation}
\vspace{-.2in} 

\begingroup

We generate synthetic datasets by \eqref{art_model}. The parameters are different between DGPs:
$\mu_i$ and $\sigma_i$ are randomly generated; the functions $h,k,l$ are linear with random coefficients; and $f_0,f_1$ is built by separated randomly initialized (then fixed) NNs.
We generate two kinds of outcome models, depending on the type of $f$: linear and nonlinear outcome models use random linear functions and NNs with invertible activations and random weights, respectively. We set the outcome and proxy noise level by $\alpha,\beta$ respectively. See Appendix \ref{sec:exp_syn_app} for more details.

We experiment on three different causal structures as shown in graphical models of Figure \ref{fig:art}, by variation on \eqref{art_model}. 
Instead of taking inputs $\rvx, \rz$ in $l$, we consider two special cases: $l\coloneqq l(\rvx)$, then $\rvx$ fully adjusts for confounding, we are in fact \textit{unconfounded}; and $l\coloneqq l(\rz)$, then we have unobserved confounder $\rz$ and \textit{proxy} $\rvx$ of $\rz$. To introduce $\rvx$ as \textit{instrumental variable}, we generate another 1-dimensional random source $\rw$ in the same way as $\rvx$, and use $\rw$ instead of $\rvx$ to generate $\rz|\rw \sim \mathcal{N}(h(\rw), \beta k(\rw))$; except indicated above, other aspects of the models are specified by \eqref{art_model}. 
% The settings correspond to ``inst.'', ``ig.'', and ``conf.'' in the legend of Figure \ref{nonl_art}. 

% satisfing ignorability to different degrees. We achieve this by tweak function $l$. First, for $l(\rvx,\rz)$, note the dimensionality of $\rvx,\rz$, we know that, on average, the influence of $\rvx$ on $\rt$ is in a sense 3 times larger than that of $\rz$. And this renders the effect of $\rz$ on $\rt$ ``partially ignorable''. Besides taking inputs $\rvx, \rz$, we also consider two special case: $l\coloneqq l(\rvx)$ which means $\rvx$ fully satisfies ignorability, and $l\coloneqq l(\rz)$ which means $\rvx$ does not improve ignorability. 

\setlength{\columnsep}{8pt}
\setlength{\intextsep}{5pt}
\begin{wrapfigure}{r}{0.36\textwidth}
\vspace{-.2in}
  \begin{center}
    \includegraphics[width=0.36\textwidth]{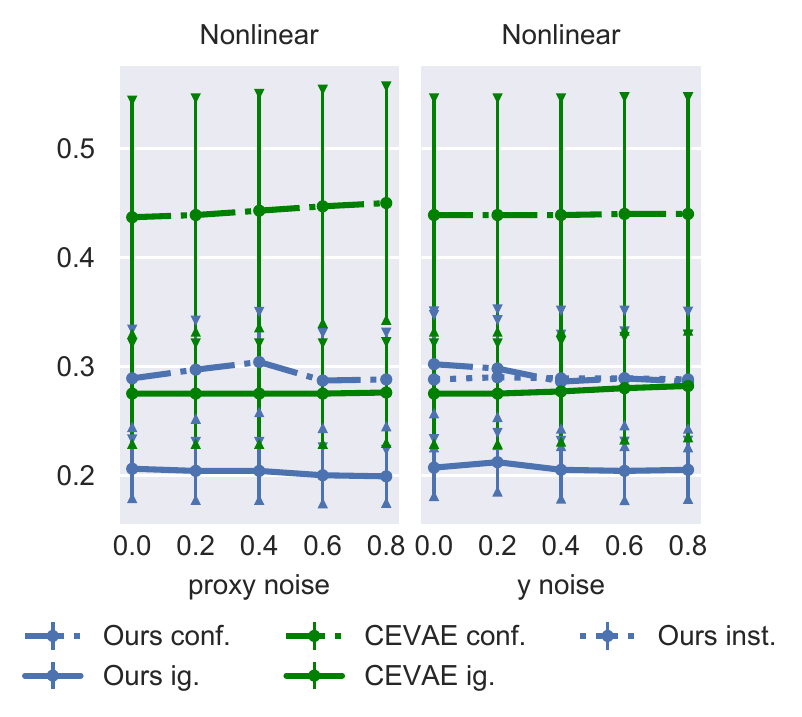}
  \end{center}
  \vspace{-.2in}
  \caption{Pre-treatment \footnotesize{$\sqrt{\epsilon_{pehe}}$ on nonlinear synthetic dataset. Error bar on 100 random DGPs. We adjust one of the noise levels $\alpha,\beta$ in each panel, with another fixed to 0.2. }}
\vspace{-.2in}
\label{nonl_art}
\end{wrapfigure}

For each causal structure, and with the same kind of outcome models, and the same noise levels ($\alpha,\beta$), we evaluate Intact-VAE and CEVAE on 100 random DGPs, with different sets of parameters in \eqref{art_model}. For each DGP, we sample 1500 data points, and split them into 3 equal sets for training, validation, and testing. Both the methods use 1-dimensional latent variable in VAE. For fair comparison, all the hyper-parameters, including type and size of NNs, learning rate, and batch size, are the same for both the methods. %\footnote{Code at \url{github.com/AMLab-Amsterdam/CEVAE}}

Figure \ref{nonl_art} shows our method significantly outperforms CEVAE on all cases
% ; CEVAE does not use conditional prior and has no theoretical guarantee in the current setting. 
Both methods work the best under unconfoundedness (``ig.''), as expected. The performances of our method on IV (``inst.'') and proxy (``conf.'') settings match that of CEVAE under unconfoundedness, showing the effective deconfounding. See Appendix \ref{sec:exp_syn_app} for results on linear outcome. Results for ATE and post-treatment are similar.

\setlength{\columnsep}{8pt}
\setlength{\intextsep}{5pt}
\begin{wrapfigure}{r}{0.4\textwidth}
\vspace{-.2in}
  \begin{center}
    \includegraphics[width=0.4\textwidth]{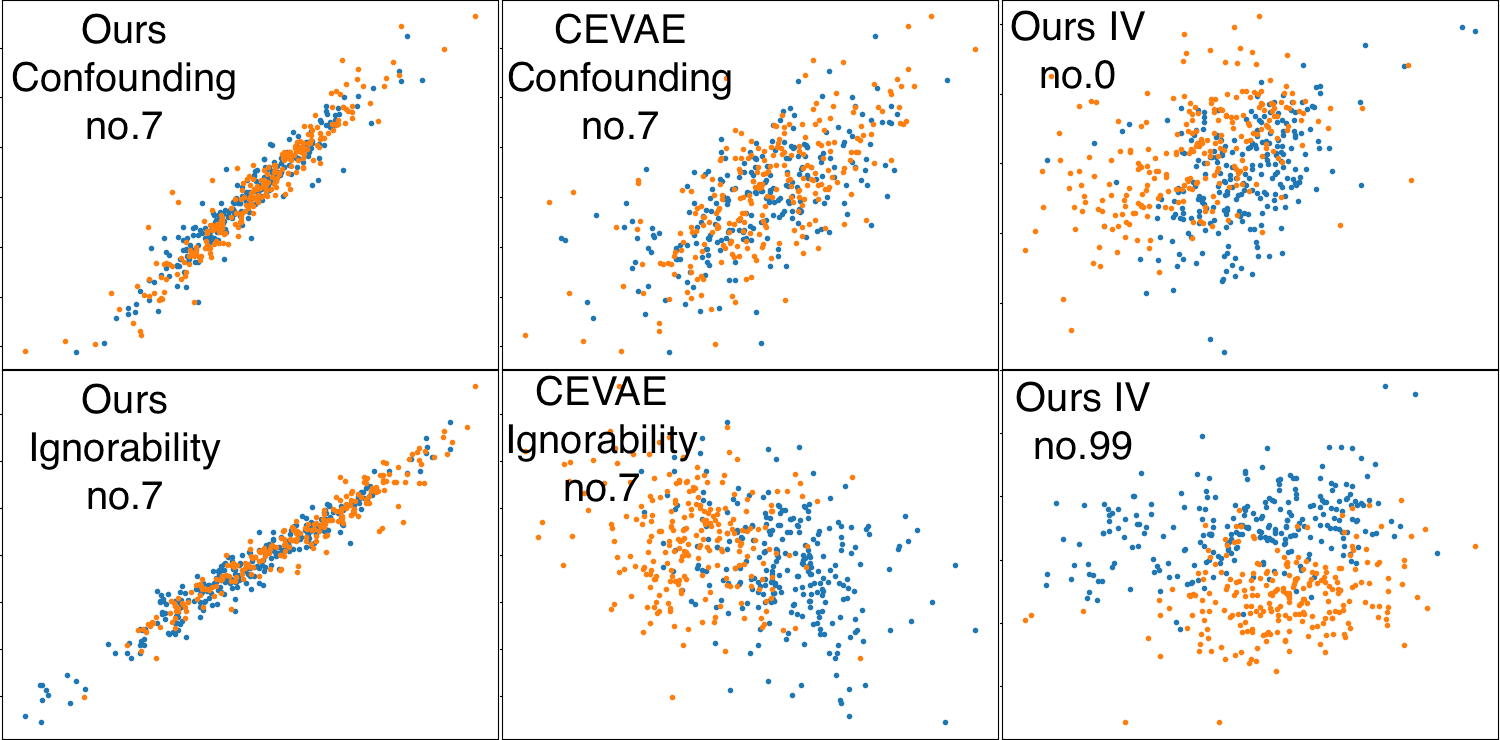}
  \end{center}
   \vspace{-.2in}
 
  \caption{\footnotesize{Plots of recovered 
  %(x) 
  - true 
  %(y) 
  latent on the nonlinear outcome. Blue: $t=0$, Orange: $t=1$. $\alpha,\beta= 0.4$. ``no.'' indicates index among the 100 random models.}} % See Appendix for more plots.
% \vspace{-.1in}
\label{recover}
\end{wrapfigure}

Here, the true latent $\rz$ is a PS, and there are no better candidate PSs than $\rz$, because $f_t$ is invertible and no information can be dropped from $\rz$. 
Thus, as shown in Figure \ref{recover}, our method learns representation as an approximate affine transformation of the true latent value, as a result of our model identifiability. More latent plots are in Appendix \ref{sec:plots} (the end of the paper). As expected, both recovery and estimation are better with unconditional prior $p(\rz|\rvx)$, and we see examples of bad recovery using conditional $p(\rz|\rvx,\rt)$ in Appendix Figure \ref{fig:bad_recover}. CEVAE shows much lower quality of recovery, particularly with large noises. Under the IV setting, while TEs are estimated as well as for the proxy setting, the relationship to the true latent is significantly obscured, because the true latent is correlated to IV $\x$ only given $\rt$, while we model it by $p(\rz|\rvx)$. This confirms that our method does not need to recover the true confounder distribution.

%We can see our method is very robust w.r.t.~both outcome and proxy noise. 
We see our method is robust to the unknown noise level.
This indicates that noises are learned by our VAE. 
%And the similar robustness of CEVAE is another side evidence.
Appendix \ref{sec:plots} shows that the noise level affects how well we recover the latent variable. %However, we should remember that recovery of true latent is a good plus, but \textit{not} essential to causal inference.

\endgroup

\subsection{IHDP benchmark dataset}
This experiment shows our balanced estimator matches the state-of-the-art methods specialized for standard unconfoundedness. The IHDP dataset \citep{hill2011bayesian} is widely used to evaluate machine learning based causal inference methods, e.g. \cite{shalit2017estimating, shi2019adapting}. 
% Unconfoundedness holds given the covariate for the IHDP. 

The generating process is as following \citep[Sec.~4.1]{hill2011bayesian}.
\begin{equation}
    \ry(0) \sim \mathcal{N}(e^{\va^T(\x+\vb)},1),\quad \ry(1) \sim \mathcal{N}(\va^T\x-o,1),
\end{equation}
where $\vb$ is constant with all elements equal to $0.5$, $\va$ and $o$ are a random parameters.
Thus, it is not necessary to condition on $\x$, but the linear PS $\va^T\x$ is sufficient. See Appendix \ref{sec:ihdp_app} for more details.
% Note, however, that this dataset violates our assumption $\rvy \independent \rvx|\rvz,\rt$, since the covariates $\rvx$ directly affect the outcome. 

To see our balancing property clearly, we add two components specialized for balancing from \cite{shalit2017estimating} into our method.
% (whose results is  shown in the caption of Table \ref{t:ihdp})
% and compare the modified and to the original. 
First, we use two separate NNs to build the two outcome functions $\vf_t(\rvz),t=0,1$ in our model \eqref{model_indep}. Second, we add to our ELBO \eqref{elbo} a regularization term, which is the Wasserstein distance \citep{cuturi2013sinkhorn} between $\E_{\data\sim p(\x|\rt=t)}p_{\blambda}(\rvz|\rvx),t\in\binset$.

As shown in Table \ref{t:ihdp}, Intact-VAE outperforms or matches the state-of-the-art methods. In particular, our method has the \textit{best} ATE estimation; and it has the \textit{best} individual-level estimation, adding the two components from \cite{shalit2017estimating}. 
We can see in the caption of Table \ref{t:ihdp}, the specialized additions do \textit{not} really improve our method, only causing a tradeoff between CATE and ATE estimation, and this may due to the tradeoff between fitting and balancing.  
And notably, our method outperforms other generative models (CEVAE and GANITE) by large margins.

We find higher than 1-dimensional $\z$ in Intact-VAE gives better results, because we have \textit{discrete} true PS due to the existence of discrete covariates.
We report results with 10-dimensional latent variable. The robustness of VAE under model misspecification was also observed by \cite{louizos2017causal}, where they used 5-dimensional Gaussian latent variable to model a binary ground truth.

\begin{table}[h]

\centering
\scriptsize
% \footnotesize
% \small
% \vspace{-.1in}
\caption{Errors on IHDP. The mean and std are calculated over 1000 random data generating models. *Results with modifications are $\epsilon_{ate}=.31_{\pm .01}/.30_{\pm .01}$ and $\sqrt{\epsilon_{pehe}}=\textbf{.77}_{\pm .02}/\textbf{.69}_{\pm .02}$. \textbf{Bold} indicates method(s) that are \textit{significantly} better than all the others. The results of the others are taken from \cite{shalit2017estimating}, except GANITE \citep{yoon2018ganite} and CEVAE \citep{louizos2017causal}.} 

\vspace{-.05in}
% \begin{center}
\begin{tabular}{p{.8cm}p{1.cm}p{1.45cm}p{1.45cm}p{1.45cm}p{1.45cm}p{1.45cm}p{1.6cm}}
\toprule 
Method &\textbf{TMLE} &\textbf{BNN} &\textbf{CFR} &\textbf{CF} &\textbf{CEVAE} &\textbf{GANITE} &\textbf{Ours*}\\
\midrule 
$\epsilon_{ate}$ &NA/.30$_{\pm .01}$ &.42$_{\pm .03}$/.37$_{\pm .03}$  
&.27$_{\pm.01}$/.25$_{\pm.01}$  &.40$_{\pm.03}$/\textbf{.18}$_{\pm.01}$ &.46$_{\pm.02}$/.34$_{\pm.01}$   &.49$_{\pm.05}$/.43$_{\pm.05}$  &$.31_{\pm .01}$/$.30_{\pm .01}$
\\
\midrule 
$\sqrt{\epsilon_{pehe}}$ &NA/5.0$_{\pm .2}$ &2.1$_{\pm .1}$/2.2$_{\pm .1}$  
&\textbf{.76}$_{\pm.02}$/\textbf{.71}$_{\pm.02}$ &3.8$_{\pm.2}$/3.8$_{\pm.2}$ &2.6$_{\pm.1}$/2.7$_{\pm.1 }$  &2.4$_{\pm .4}$/1.9$_{\pm .4}$ &\textbf{.77}$_{\pm .02}$/\textbf{.69}$_{\pm .02}$
\\
\bottomrule 
\end{tabular}

% \begin{tabular}{p{1.cm}p{.5cm}p{.5cm}p{.5cm}p{.5cm}p{.5cm}p{.6cm}p{.6cm}}
% \toprule 
% Method &{TMLE} &{BNN} &{CFR} &{CF} &{CEVAE} &{GANITE} &{Ours*}\\
% \midrule 
% pre-$\epsilon_{ate}$ &NA &.42$_{\pm .03}$
% &.27$_{\pm.01}$ &.40$_{\pm.03}$ &.46$_{\pm.02}$ &.49$_{\pm.05}$  &\textbf{.21}$_{\pm .01}$
% \\
% \midrule 
% post-$\epsilon_{ate}$ &.30$_{\pm .01}$ &.37$_{\pm .03}$  
% &.25$_{\pm.01}$  &\textbf{.18}$_{\pm.01}$ &.34$_{\pm.01}$   &.43$_{\pm.05}$  &\textbf{.17}$_{\pm .01}$
% \\
% \midrule 
% pre-$\sqrt{\epsilon_{pehe}}$ &NA &2.1$_{\pm .1}$ 
% &\textbf{.76}$_{\pm.02}$ &3.8$_{\pm.2}$ &2.6$_{\pm.1}$ &2.4$_{\pm .4}$ &1.0$_{\pm .05}$
% \\
% \midrule 
% post-$\sqrt{\epsilon_{pehe}}$ &5.0$_{\pm .2}$ &2.2$_{\pm .1}$  
% &\textbf{.71}$_{\pm.02}$ &3.8$_{\pm.2}$ &2.7$_{\pm.1 }$  &1.9$_{\pm .4}$ &.97$_{\pm .04}$
% \\
% \bottomrule 
% \end{tabular}
\label{t:ihdp}
% \end{center}
% \vspace{-.1in}
\end{table}

\subsection{Pokec social network dataset}
We show our method is the best compared with the methods specialized for networked deconfounding, a challenging problem in its own right. Pokec \citep{leskovec2014snap} is a real world social network dataset. We experiment on a semi-synthetic dataset based on Pokec, introduced in \cite{veitch2019using}, and use exactly the same pre-processing and generating procedure. The pre-processed network has about 79,000 vertexes (users) connected by 1.3 $\times 10^6$ undirected edges.  The subset of users used here are restricted to three living districts that are within the same region. The network structure is expressed by binary adjacency matrix $\mG$. 
 %\footnote{Code at \url{github.com/vveitch/causal-network-embeddings}}

Each user has 12 attributes, among which  \texttt{district}, \texttt{age}, or \texttt{join date} is used as a confounder $\rz$ to build 3 different datasets, with remaining 11 attributes used as covariate $\rvx$. Treatment $\rt$ and outcome $\rvy$ are synthesised as following:
% \vspace{-.1in} 
\begin{equation}
\label{pokec}
    \rt \sim \bern(g(\rz)), \quad \rvy = \rt + 10(g(\rz)-0.5) + \epsilon, \text{ where $\epsilon$ is standard normal. }
    % \epsilon \sim \mathcal{N}(0, 1)
% \vspace{-.1in} 
\end{equation}
Note that \texttt{district} is of 3 categories; \texttt{age} and \texttt{join date} are also discretized into three bins. There is a PS that is $g(\rz)$, which maps the three categories and values to $\{0.15, 0.5, 0.85\}$. 
% The parameter $\beta$ (set to 1 or 10) controls the strength of confounding, and also the relative noise level. 

As in \cite{veitch2019using}, we split the users into 10 folds, test on each fold and report the mean and std of pre-treatment ATE predictions. We further separate the rest of users (in the other 9 folds) by 6:3, for training and validation. Table \ref{tbl:Pokec} shows the results. In addition, the pre-treatment $\sqrt{\epsilon_{pehe}}$ for \texttt{Age}, \texttt{District}, and \texttt{Join date} confounders are 1.085, 0.686, and 0.699 respectively, practically the same as the ATE errors. \cite{veitch2019using} do not give individual-level prediction. 

\begin{table}[h]
\centering
\scriptsize

% \vspace{-.1in}
\caption{Pre-treatment ATE on Pokec. Ground truth ATE is 1, as we can see in \eqref{pokec}. ``Unadjusted'' estimates ATE by $\E_{\mathcal{D}}(y_1)-\E_{\mathcal{D}}(y_0)$. ``Parametric'' is a stochastic block model for networked data \citep{gopalan2013efficient}. ``Embed-'' denotes the best alternatives given by \citep{veitch2019using}. \textbf{Bold} indicates method(s) that are \textit{significantly} better than all the others. 20-dimensional latent variable in Intact-VAE works better, and its result is reported. The results of the other methods are taken from \citep{veitch2019using}.} \label{tbl:Pokec}

\vspace{-.05in}
\begin{tabular}{cccccc}
\toprule
{}                & {Unadjusted} & {Parametric} & {Embed-Reg.} & {Embed-IPW} & {Ours}                  \\
\midrule
\texttt{Age}       & 4.34 $\pm$ 0.05                & 4.06 $\pm$ 0.01                & 2.77 $\pm$ 0.35                    & 3.12 $\pm$ 0.06                   & \textbf{2.08} $\pm$ 0.32                                \\
\texttt{District}  & 4.51 $\pm$ 0.05                & 3.22 $\pm$ 0.01                & \textbf{1.75} $\pm$ 0.20                    & \textbf{1.66} $\pm$ 0.07                   & \textbf{1.68} $\pm$ 0.10                           \\
\texttt{Join Date} & 4.03 $\pm$ 0.06                & 3.73 $\pm$ 0.01                & 2.41 $\pm$ 0.45                    & 3.10 $\pm$ 0.07                   & \textbf{1.70} $\pm$ 0.13 \\
\bottomrule
\end{tabular}
% \vspace{-.2in}
\end{table}

Intact-VAE is expected to learn a PS from $\mG, \rvx$, if we can exploit the network structure effectively. 
% The important challenges are 1) $\rvx$ obviously does not satisfy ignorability, and 2) large outcome noise exists.  %particularly for Intact-VAE. 
% With $\beta=1$, we have very large noise relative to the range of $\rvy$, and the noise is still non-negligible when $\beta=10$. % 2) the latent variable is categorical [misspecified latent model], 
Given the huge network structure, most users can practically be identified by their attributes and neighborhood structure, which means $\rz$ can be roughly seen as a deterministic function of $\mG, \rvx$. 
% Then, $\mG, \rvx$ can be, 
% as defined by us, \textit{noiseless proxies} of $\rz$ (see Appendix 8.5). 
This idea is comparable to Assumptions 2 and 4 in \cite{veitch2019using}, which postulate directly that a balancing score can be learned in the limit of infinite large network. %Our idea based on the concept of noiseless proxy is more naturally justified, and 
%he learning of balancing score is then a result of Theorem \ref{id_tion1}.
% To extract information from the network structure, we use Graph Convolutional Network (GCN) \citep{DBLP:conf/iclr/KipfW17} in conditional prior and encoder of Intact-VAE. 

To extract information from the network structure, we use Graph Convolutional Network (GCN) \citep{DBLP:conf/iclr/KipfW17} in the prior and encoder of Intact-VAE. 
% The implementation details are given in Appendix.
Note that GCN cannot be trained by mini-batch, instead, we perform batch gradient decent using all data for each iteration, with initial learning rate $10^{-2}$. We use dropout \citep{srivastava2014dropout} with rate 0.1 to prevent overfitting.

GCN need to take as inputs the network matrix $\mG$ and the covariates matrix $\mX \coloneqq (\vx_1^T,\dotsc,\vx_M^T)^T$ of \textit{all} users, where $M$ is user number, regardless of whether it is in training, validation, or testing phase; and it outputs a representation matrix $\mR$, again for all users. To enable sample separation, we need to make sure the treatment and outcome are used only in the respective phase, e.g., $(y_m,t_m)$ of a testing user $m$ is only used in testing. During training, we select the rows in $\mR$ that correspond to users in training set. Then, treat this \textit{training representation matrix} as if it is the covariate matrix for a non-networked dataset, that is, the downstream networks in conditional prior and encoder are the same as in the other two experiments, except that they take $(\mR_{m,:})^T$ where $\vx_m$ was expected as input. And we have respective selection operations for validation and testing. We can still train Intact-VAE including GCN by Adam, simply setting the gradients of non-seleted rows of $\mR$ to 0.

\section{Conclusion}

In this work, we proposed Intact-VAE for CATE estimation. Our generative model is identifiable and has a sufficient score as its latent variable. Our method outperforms or matches state-of-the-art methods under diverse settings including unobserved confounding. In Sec.~\ref{sec:principled}, we explained why the current VAE methods are unsatisfactory from a more ``causal'' viewpoint. We gave theoretical analysis of Intact-VAE under unconfoundedness and discussed parallel and future theoretical work--TE identification without overlap and approaches to identification and optimal estimation under unobserved confounding. We believe this series of work will also pave the way towards principled causal effect estimation by other deep architectures, given the fast advances in deep identifiable models. 
For example, recently, \cite{khemakhem2020ice} provide identifiability to deep energy models, and \cite{roeder2020linear} extend the result to a wide class of state-of-the-art deep discriminative models. We hope this work will inspire other methods based on deep identifiable models.

\bibliography{example_paper}

% %%%%%%%%%%%%%%%%%%%%%%%%%%%%%%%%%%%%%%%%%%%%%%%%%%%%%%%%%%%%%%%%%%%%%%%%%%%%%%%
% %%%%%%%%%%%%%%%%%%%%%%%%%%%%%%%%%%%%%%%%%%%%%%%%%%%%%%%%%%%%%%%%%%%%%%%%%%%%%%%
% % DELETE THIS PART. DO NOT PLACE CONTENT AFTER THE REFERENCES!
% %%%%%%%%%%%%%%%%%%%%%%%%%%%%%%%%%%%%%%%%%%%%%%%%%%%%%%%%%%%%%%%%%%%%%%%%%%%%%%%
% %%%%%%%%%%%%%%%%%%%%%%%%%%%%%%%%%%%%%%%%%%%%%%%%%%%%%%%%%%%%%%%%%%%%%%%%%%%%%%%
% \appendix
\clearpage

\appendix

\section{Proofs and additional theoretical results}
\label{sec:proofs}

% Here, we give proofs of results presented in the main text. More theoretical results and their proof can be found in ``Theoretical exposition''.

% By a slight abuse of symbol, we will overload $|$ to collect the dependence on $\rt=t$, e.g., \eqref{estimator} can be written as $    \mu(\vx)=\E_{\mathcal{D}|\vx}(\vf'(\vz_{\bm\phi'}))|t$. 

Theorem \ref{cate_by_bts} is rather straightforward from \eqref{eq:ps_indp} (see Lemma \ref{prop:pscore}) and the definition of PtS, and thus its proof is omitted.

We give main properties of Pt-score as following. 
\begin{lemma}
\label{prop:pscore}
If $\rvv$ gives exchangeability, and $\PS_{\rt}(\rvv)$ is a Pt-score, then $\rvy(t)\independent \rvv,\rt|\PS_t(\rvv)$.
\end{lemma}

The following three properties of conditional independence will be used repeatedly in proofs.
\begin{proposition}[Properties of conditional independence]
\label{indep_prop}
\citep[1.1.55]{pearl2009causality} For random variables $\rvw, \rvx, \rvy, \rvz$. We have:
\begin{equation*}
    \begin{split}
        \rvx \independent \rvy|\rvz \land \rvx \independent \rvw|\rvy, \rvz &\implies \rvx \independent \rvw,\rvy|\rvz \text{ (Contraction)}. \\
        \rvx \independent \rvw,\rvy|\rvz &\implies \rvx \independent \rvy|\rvw,\rvz \text{ (Weak union)}. \\ 
        \rvx \independent \rvw,\rvy|\rvz &\implies \rvx \independent \rvy|\rvz \text{ (Decomposition)}.
    \end{split}
\end{equation*}
\end{proposition}

\begin{proof}[of Lemma \ref{prop:pscore}]
From $\rvy(t)\independent\rt|\rvv$ (\textit{exchangeability} of $\rvv$), and since $\PS_t$ is a \textit{function} of $\rvv$, we have $\rvy(t)\independent\rt|\PS_t(\rvv),\rvv$ (1).

From (1) and $\rvy(t)\independent\rvv|\PS_t(\rvv)$ (definition of Pt-score), using contraction rule, we have \\ $\rvy(t)\independent\rt,\rvv|\PS_t(\rvv)$ for both $t$. 
% From \eqref{eq:ps_indp}, using weak union rule, we have $\rvy(t)\independent\rvv|\PS_t,\rt$. Let $\rt=t$ in the condition, and note $\rvy(t)=\rvy$ if $\rt=t$ (consistency), we have $\rvy\independent\rvv|\PS_t,\rt=t$. This holds for both $t$, we have property 1).
% Also from \eqref{eq:ps_indp}, using decomposition rule, we have $\rvy(t)\independent\rt|\PS_t$ for both $t$. This is 2).
\end{proof}
Apply the proposition to our setting, we have eq.~\eqref{eq:ps_indp}.

In the proof of Theorem \ref{idmodel}, all equations and variables should condition on $t$, and we omit the conditioning in notation for convenience. 

The main part of our model identifiability follows from that of Theorem 1 in \cite{khemakhem2020variational}, but now adapted to the dependency on $t$. Here we give an outline of the proof, and the details can be easily filled by referring to \cite{khemakhem2020variational}.
% and its alternative formulation 

\begin{proof}[of Theorem \ref{idmodel}]

Using i) and ii) , we transform $p_{\vf,\bm\lambda}(\rvy|\rvx,t)=p_{\vf',\bm\lambda'}(\rvy|\rvx,t)$ into equality of noiseless distributions, that is, 
\begin{equation}
    q_{\vf',\bm\lambda'}(\rvy)=q_{\vf,\bm\lambda}(\rvy)\coloneqq p_{\bm\lambda}(\vf^{-1}(\rvy)|\rvx,t)vol(\mJ_{\vf^{-1}}(\rvy))\mathbb{I}_{\mathcal{Y}}(\rvy)
\end{equation}
where $p_{\bm\lambda}$ is the Gaussian density function of the conditional prior defined in \eqref{model_indep} and $vol(A)\coloneqq\sqrt{\det A A^T}$.
$q_{\vf',\bm\lambda'}$ is defined similarly to $q_{\vf,\bm\lambda}$.

Then, plug \eqref{model_indep} into the above equation, and take derivative on both side at the $\vx^o$ in \ref{ass:inv_jac}, we have
\begin{equation}
    % \mathcal{F'}(\rvy)=\mathcal{F}(\rvy)\coloneqq\mJ_{\bm\lambda}(\vx^o)^T\vt(\vf^{-1}(\rvy))-\nabla\alpha(\vx^o)
    \mathcal{F'}(\rvy)=\mathcal{F}(\rvy)\coloneqq\mL^T\vt(\vf^{-1}(\rvy))-\bbeta
\end{equation}
where $\vt(\rvz)\coloneqq(\rvz,\rvz^2)^T$ is the sufficient statistics of factorized Gaussian, and $\bbeta_t\coloneqq(\alpha_t(\vx_1)-\alpha_t(\vx_0),...,\alpha_t(\vx_{2n})-\alpha_t(\vx_0))^T$ where $\alpha_t(\rvx;\bm\lambda_t)$ is the log-partition function of the conditional prior. $\mathcal{F'}$ is defined similarly to $\mathcal{F}$, but with $\vf',\bm\lambda',\alpha'$

Since $\mL$ is invertible, we have 
\begin{equation}
    \vt(\vf^{-1}(\rvy))=\mA\vt(\vf'^{-1}(\rvy))+\vc
\end{equation}
% where $\mA = \mJ_{\bm\lambda}(\vx^o)^{-T}\mJ_{\bm\lambda'}(\vx^o)^{T}$ and $\vc = \mJ_{\bm\lambda}(\vx^o)^{-T}\nabla(\alpha(\vx^o)-\alpha'(\vx^o))$. 
where $\mA = \mL^{-T}\mL'^{T}$ and $\vc = \mL^{-T}(\bbeta-\bbeta')$.

The final part of the proof is to show, by following the same reasoning as in Appendix B of \cite{sorrenson2019disentanglement}, that $\mA$ is a sparse matrix such that
\begin{equation}
    \mA=\begin{pmatrix}
    % \mA^{11} & \mA^{12} \\
    % \mA^{21} & \mA^{22}
    \diag(\va) & \mO \\
    \diag(\vu) & \diag(\va^2)
    \end{pmatrix}
    % \mA^{11}=\diag(\va),\mA^{11}=\diag(\va^2),\mA^{21}_t=\diag(\vc),\mA^{12}_t=\mO
\end{equation}
where $\mA$ is partitioned into four $n$-square matrices. Thus
\begin{equation}
    \vf^{-1}(\rvy)=\diag(\va)\vf'^{-1}(\rvy)+\vb
\end{equation}
where $\vb$ is the first half of $\vc$.
\end{proof}

\begin{proof}[of Theorem \ref{th:id_os}]
Under i), and because $\vf_t$ is injective, we have 
\begin{equation}
\label{eq:mean_match}
    \E_{p_{\btheta}}(\y|\x,\rt)=\E(\y|\x,\rt) \implies \vf_{t}\circ\vh(\vx)=\vj_{t}\circ\OS(\vx) \text{ on $(\vx,t)$ such that $p(t,\vx) > 0$}.
\end{equation}
We show the solution set of (\theequation) is
\begin{equation}
    \{(\vf,\vh)|\vf_t=\vj_t\circ\Delta\inv,\vh=\Delta\circ\OS,\Delta: \mathcal{P} \to \R^n \text{ is injective}\}.
\end{equation}
% where $\mathcal{P}$ is the support of $\OS$. 

By i) and ii), with injective $\vf_t,\vj_t$ and $\dim(\z)=\dim(\y) \geq \dim(\OS)$, for any $\Delta$ above, there exists a functional parameter $\vf_t$ such that $\vj_t=\vf_t\circ\Delta$. Thus, set (\theequation) is non-empty, and any element is indeed a solution because $\vf_t\circ\vh=\vj_t\circ\Delta\inv\circ\Delta\circ\OS=\vj_t\circ\OS$.

Any solution of \eqref{eq:mean_match} should be in (\theequation). A solution should satisfy $\vh(\vx)=\vf_{t}\inv\circ\vj_{t}\circ\OS(\vx)$ for both $t$ since $\vx$ is overlapping. This means the \textit{injective} function $\vf_{t}\inv\circ\vj_{t}$ should \textit{not} depend on $t$, thus it is one of the $\Delta$ in (\theequation). 

We proved conclusion 1) with $\vv\coloneqq\Delta$. And conclusion 2) is quickly seen from
\begin{equation}
    \hat{\mu}_{t}(\vx)=\vf_t(\vh(\vx))=\vj_t\circ\vv\inv(\vv\circ\OS(\vx))=\vj_{t}(\OS(\vx))=\mu_{t}(\vx).
\end{equation}
\end{proof}

% \subsection{Consistency of VAE}
% \label{sec:consist}
The following is a refined version of Theorem 4 in \cite{khemakhem2020variational}.
The result is proved by assuming: i) our VAE is flexible enough to ensure the ELBO is tight (equals to the true log likelihood) for some parameters; ii) the optimization algorithm can achieve the \textit{global} maximum of ELBO (again equals to the log likelihood).
\begin{proposition}[Consistency of VAE]
\label{consistency}
Given the VAE \eqref{model_indep}--\eqref{more_param}, and let $p^*(\rvx,\rvy,\rt)$ be the true observational distribution, assume 

\renewcommand{\labelenumi}{\roman{enumi})}
\begin{enumerate}
\def\theenumi{\roman{enumi})}
    \item there exists $(\bar{\vtheta}, \bar{\bm\phi})$ such that $p_{\bar{\vtheta}}(\rvy|\rvx,\rt)=p^*(\rvy|\rvx,\rt)$ and $p_{\bar{\vtheta}}(\rvz|\rvx,\rvy,\rt)=q_{\bar{\bm\phi}}(\rvz|\rvx,\rvy,\rt)$;
    
    \item \label{ass:prime} the ELBO $\E_{\mathcal{D} \sim p^*}(\mathcal{L}(\rvx,\rvy,\rt; \vtheta, \bm\phi))$ \eqref{elbo} can be optimized to its global maximum at $(\vtheta', \bm\phi')$;
\end{enumerate}

Then, in the limit of infinite data, $p_{\vtheta'}(\rvy|\rvx,\rt)=p^*(\rvy|\rvx,\rt)$ and $p_{\vtheta'}(\rvz|\rvx,\rvy,\rt)=q_{\bm\phi'}(\rvz|\rvx,\rvy,\rt)$.
\end{proposition}

% Both assumptions are widely assumed in practice, and together ensure $p_{\vtheta'}=p^*$. 
\begin{proof}
From i), we have $\mathcal{L}(\rvx,\rvy,\rt; \bar{\vtheta}, \bar{\bm\phi})=\log p^*(\rvy|\rvx,\rt)$. But we know $\mathcal{L}$ is upper-bounded by $\log p^*(\rvy|\rvx,\rt)$. So, $\E_{\mathcal{D} \sim p^*}(\log p^*(\rvy|\rvx,\rt))$ should be the global maximum of the ELBO (even if the data is finite).

Moreover, note that, for any $(\vtheta, \bm\phi)$, we have $\KL(p_{\vtheta}(\rvz|\rvx,\rvy,\rt) \Vert q_{\bm\phi}(\rvz|\rvx,\rvy,\rt) \geq 0$ and, in the limit of infinite data, $\E_{\mathcal{D} \sim p^*}(\log p_{\vtheta}(\rvy|\rvx,\rt)) \leq \E_{\mathcal{D} \sim p^*}(\log p^*(\rvy|\rvx,\rt))$. Thus, the global maximum of ELBO is achieved \textit{only} when $p_{\vtheta}(\rvy|\rvx,\rt)=p^*(\rvy|\rvx,\rt)$ and $p_{\vtheta}(\rvz|\rvx,\rvy,\rt)=q_{\bm\phi}(\rvz|\rvx,\rvy,\rt)$.
\end{proof}

Based on this, consistent prior estimation of CATE follows directly from TE identification. The following is a corollary of Theorem \ref{th:id_os}.
% , and is again independent of Theorem \ref{idmodel}.
\begin{corollary}
\label{th:estimation}
Under the conditions of Theorem \ref{th:id_os}, 
% let $\bm\sigma_{\vy} = \bm0$, and 
% \begin{inparaenum}[i]
% \def\theenumi{\roman{enumi})}
further require the consistency of Intact-VAE.
% under model \eqref{model_indep}\&\eqref{eq:enc}, trained by ELBO \eqref{elbo}.
% which implies $p_{\vtheta'}(\vz|\vx,\vy,\rt)=q_{\bm\phi'}(\vz|\vx,\vy,\rt)$ where 
% ii) $\vs'(\vx,\vy,\rt)$ equals to the variance of $p_{\vtheta'}(\vz|\vx,\vy,\rt)$,
% \end{inparaenum}
Then, in the limit of infinite data, we have $\mu_{t}(\x)=\vf_{t}(\vh_{t}(\x))$
% $\mathcal{D}\coloneqq\{(\x,\y,\rt)\}\sim p\st(\x,\y,\rt)$, 
% \vspace{-.1in}
where $\vf,\vh$ are the optimal parameters learned by the VAE.
\end{corollary}
% [y for E1]

% \begin{proof} % We proceed similarly to Corollary 1. From the consistency, $\rvz$ should have the same support as $\rvz'$. 
% In the limit of $\bm\sigma_{\rvz,t} \to \bm0$, we have $\rvz=\vh(\rvx)$ and $\rvz'=\vh'(\rvx)$. For all $\vx$, there exist a \textit{unique} $\vz=\vh(\vx)$ and a \textit{unique} $\vz'=\vh'(\vx)$. From the positivity of $\rvz$, $\rvz_0$ and $\rvz_1$ have the same support. Substitute $\vy=\vf(\vz)$ into the l.h.s of \eqref{eq:class}, and $\vy=\vf'(\vz')$ into the r.h.s, we have $\vz = \diag(\va)\vz' + \vb$. The relation is one-to-one for all $\vz$, so we get Corollary \ref{idrepre}. 
% \end{proof}

\section{Additional backgrounds}

% We say that \textit{weak ignorability} holds when we have both exchangeability and positivity.

% Exchangeability means there is no correlation between factual assignment of treatment and counterfactual outcomes, given $\rconf,\rvx$, just as it is the case in an RCT. 
% Thus, it can be understood as \textit{unconfoundedness} given $\rconf,\rvx$, and $\rconf$ can be seen as the unobserved confounder(s). Positivity says the supports of $p(\rt=t|\rconf,\rvx)$ should always be \textit{overlapped}, and this ensures there are no impossible events in the conditions after adding $\rt=t$.
% % , and the expectations can thus be estimated from observational data. 
% Finally, consistent counterfactuals are well defined. In causal inference, we see $\rvy(t)$ as the underlying hidden variables that give \textit{factual} (observational) $\y$ if $\rt=t$ is assigned. If, say, we assigned treatment $\rt=1$, we observe $\vy=\vy(1)$, a realization of $\y(1)$. We understand that, there exists also $\y(0)$ corresponding to the random outcome we would observe, if we had applied $\rt=0$, the counterfactual assignment. 
% shows a detailed derivation of \eqref{eq:id}.

\subsection{Prognostic score and balancing score}
\label{sec:scores}

In the fundamental work of \citep{hansen2008prognostic}, prognostic score is defined equivalently to our P0-score, but it in addition requires no effect modification to work for $\y(1)$. Thus, a useful prognostic score corresponds to our Pt-score. 
Note particularly, Lemma \ref{prop:pscore} implies $\rvy(t)\independent \rt|\PS_t(\rvv)$ (using decomposition rule). Thus, if $\rvv$ gives weak ignorability (exchangeability plus positivity) and $\PS(\rvv)$ is a P-score, then $\PS$ also gives weak ignorability, which is a nice property shared with balancing score, as we will see immediately.

Prognostic scores are closely related to the important concept of balancing score \citep{rosenbaum1983central}. 
\begin{definition}[Balancing score]
\label{bscore}
$\bm\beta(\rvv)$, a function of random variable $\rvv$, is a balancing score if $\rt \independent \rvv|\bm\beta(\rvv)$.
\end{definition}
\begin{proposition}
Let $\bm\beta(\rvv)$ be a function of random variable $\rvv$. $\bm\beta(\rvv)$ is a balancing score if and only if $f(\bm\beta(\rvv))=p(\rt=1|\rvv)\coloneqq e(\rvv)$ for some function $f$ (or more formally, $e(\rvv)$ is $\bm\beta(\rvv)$-measurable). Assume further that $\rvv$ gives weak ignorability, then so does $\bm\beta(\rvv)$.
\end{proposition}
Obviously, the \textit{propensity score} $e(\rvv):=p(\rvt=1|\rvv)$, the propensity of assigning the treatment given $\rvv$, is a balancing score (with $f$ be the identity function). Also, given any invertible function $\vv$, the composition $\vv \circ \bm\beta$ is also a balancing score since $f\circ \vv^{-1}(\vv \circ \bm\beta(\rvv))=f(\bm\beta(\rvv))=e(\rvv)$.

Compare the definition of balancing score and prognostic score, we can say balancing score is sufficient for the treatment $\rt$ ($\rt \independent \rvv|\bm\beta(\rvv)$), while prognostic score (Pt-score) is sufficient for the potential outcomes $\y(t)$ ($\y(t) \independent \rvv|\PS_t(\rvv)$). They complement each other; conditioning on either deconfounds the potential outcomes from treatment, with the former focuses on the treatment side, the latter on the outcomes side.

\subsection{VAE, Conditional VAE, and iVAE}
\label{vaes}

VAEs \citep{kingma2019introduction} are a class of latent variable models with latent variable $\rvz$, and observable $\rvy$ is generated by the decoder $p_\vtheta(\rvy|\rvz)$.
In the standard formulation \citep{DBLP:journals/corr/KingmaW13}, the variational lower bound $\mathcal{L}(\rvy;\vtheta,\bm\phi)$ of the
log-likelihood is derived as:
% \vspace{-.1in}
\begin{equation}
\label{elbo_vae}
\begin{split}
      \log p(\rvy) &\geq \log p(\rvy) - \KL(q(\rvz|\rvy)\Vert p(\rvz|\rvy)) \\ 
      &= \E_{\vz \sim q}\log p_\vtheta(\rvy|\vz) - \KL(q_{\bm\phi}(\rvz|\rvy)\Vert p(\rvz)),
    %   = \underbrace{\E_{\vz \sim q}\log p_\vtheta(\rvy|\vz) - \KL(q_{\bm\phi}(\rvz|\rvy)\Vert p(\rvz))}_{\mathcal{L}_{VAE}(\rvy;\vtheta,\bm\phi)}, 
\end{split}
% \vspace{-.1in}
\end{equation}
where $\KL$ denotes KL divergence and the encoder $q_{\bm\phi}(\rvz|\rvy)$ is introduced to approximate the true posterior $p(\rvz|\rvy)$. The decoder $p_\vtheta$ and encoder $q_{\bm\phi}$ are usually parametrized by NNs. We will omit the parameters $\vtheta,{\bm\phi}$ in notations when appropriate.

The parameters of the VAE can be learned with stochastic gradient variational Bayes. 
With Gaussian latent variables, the KL term of $\mathcal{L}$ has closed form, while the first term can be evaluated by drawing samples from the approximate posterior $q_{\bm\phi}$ using the reparameterization trick \citep{DBLP:journals/corr/KingmaW13}, then, optimizing the evidence lower bound (ELBO) $\E_{\rvy \sim \mathcal{D}}(\mathcal{L}(\vy))$ with data $\mathcal{D}$, we train the VAE efficiently. 

Conditional VAE (CVAE) \citep{sohn2015learning,kingma2014semi} adds a conditioning variable $\rvc$, usually a class label, to standard VAE (See Figure \ref{f:vae}).
With the conditioning variable, CVAE can give better reconstruction of each class. The variational lower bound is
\begin{equation}
    \log p(\rvy|\rvc) \geq \E_{\vz \sim q}\log p(\rvy|\vz,\rvc) - \KL(q(\rvz|\rvy,\rvc)\Vert p(\rvz|\rvc)). 
    % \coloneqq \mathcal{L}_{CVAE}(\rvy,\rvc) 
\end{equation}
The conditioning on $\rvc$ in the prior is usually omitted \citep{doersch2016tutorial}, i.e., the prior becomes $\rvz \sim \mathcal{N}(\bm0, \mI)$ as in standard VAE, since the dependence between $\rvc$ and the latent representation is also modeled in the encoder $q$. Moreover, unconditional prior in fact gives better reconstruction because it encourages learning representation independent of class, similarly to the idea of beta-VAE \citep{higgins2016beta}.

As mentioned, \textit{identifiable} VAE (iVAE) \citep{khemakhem2020variational} provides the first identifiability result for VAE, using auxiliary variable $\x$. It assumes $\rvy \independent \x|\rvz$, that is, $p(\rvy|\rvz,\x)=p(\rvy|\rvz)$. The variational lower bound is
% \vspace{-.1in}
\begin{equation}
\begin{split}
    \log p(\rvy|\x) &\geq \log p(\rvy|\x) - \KL(q(\rvz|\rvy,\x)\Vert p(\rvz|\rvy,\x)) \\
    &=\E_{\vz \sim q}\log p_{\vf}(\rvy|\vz) - \KL(q(\rvz|\rvy,\x)\Vert p_{\bm T,\bm\lambda}(\rvz|\x)),
    % \underbrace{\E_{\vz \sim q}\log p_{\vf}(\rvy|\vz) - \KL(q(\rvz|\rvy,\x)\Vert p_{\bm T,\bm\lambda}(\rvz|\x))}_{\mathcal{L}_{iVAE}(\rvy,\x)} 
\end{split}
% \vspace{-.02in}
\end{equation}
where $\rvy=\vf(\rvz)+\bm\epsilon$, $\bm\epsilon$ is additive noise, and $\rvz$ has exponential family distribution with sufficient statistics $\bm T$ and parameter $\bm \lambda(\x)$. Note that, unlike CVAE, the decoder does \textit{not} depend on $\x$ due to the independence assumption.

Here, \textit{identifiability of the model} means that the functional \textit{parameters} $(\vf,\bm T,\bm\lambda)$ can be identified (learned) up to certain simple transformation. Further, in the limit of $\bm\epsilon \to \bm0$, iVAE solves the nonlinear ICA problem of recovering  $\rvz=\vf^{-1}(\rvy)$.

\subsection{CEVAE: Comparisons and Criticisms}
\label{sec:cevae}

There are few theoretical justifications for CEVAE. Their Theorem 1 directly assumes the joint distribution $p\st(\x,\y,\rt,\rconf)$ including hidden confounder $\rconf$ is recovered, then identification is trivial by using the standard adjustment equation \eqref{eq:id}. The theorem is in essence no more than giving an example where \eqref{eq:id} works. 

However, as we mentioned in Introduction and Sec.~\ref{sec:setup}, the challenge is exactly that the confounder is hidden, unobserved. Many years of work was done in causal inference, to derive conditions under which hidden confounder can be (partially) recovered \citep{greenland1980effect, kuroki2014measurement, miao2018identifying}. In particular, \cite{miao2018identifying} gives the most recent identification result for proxy setting, which requires very specific two proxies structure, and other completeness assumptions on distributions. Thus, it is unreasonable to believe that VAE, with simple descendant proxies, can recover the hidden confounder. 

Moreover, the identifiability of VAE itself is a challenging problem. As mentioned in Introduction and Sec.~\ref{vaes}, \cite{khemakhem2020variational} is the first identifiability result for VAE, but it only identifies equivalence class, not a unique representation function. Thus, it is also unconvincing that VAE can learn a unique latent distribution, without certain assumptions. 

As we show in Sec.~5.1, for relatively simple synthetic dataset, CEVAE can not robustly recover the hidden confounder, even only up to transformation, while our method can (though, again, this is not needed for our method). 

Our work aims to fill this gap of justification for VAE methods, see the next subsection for more. Below we give some straightforward difference between our method and CEVAE.

\paragraph{Motivation}
Our method is motivated by the sufficient scores. In particular, our method is motivated by prognostic scores \citep{hansen2008prognostic}, and our model is directly based on equations \eqref{eq:cate_by_bts} which identifies CATE from PtSs. There is no need to recover the hidden confounder in our framework.

CEVAE is motivated by exploiting proxy variables, and its intuition is that the hidden confounder $\rconf$ can be recovered by VAE from proxy variables. 

\paragraph{Applicability}
As a result, proxy variable $\x$ is contained as a special case as shown in our Figure \ref{fig:setting}. 
CEVAE assumes a specific structure among the variables (their Figure 1). In particular, their covariate $\x$, 1) can only contain descendant proxies, 2) cannot affect the outcome directly, and 3), as implicitly assumed in their (2) for decoder, cannot affect the treatment also. That is, their problem setting is just our Figure \ref{fig:setting} with only one possibility $\x=\x_{p_d}$. 

\paragraph{Architecture}
Our model is naturally based on \eqref{eq:cate_by_bts}, particularly the independence properties of PtS. And as a consequence, our VAE architecture is a natural combination of iVAE and CVAE (see Figure \ref{f:vae}). Our ELBO \eqref{elbo} is derived by standard variational lower bound. 

On the other hand, the architecture of CEVAE is more ad hoc and complex. Its decoder follows the graphical model of descendant proxy mentioned above, but adds an ad hoc component to mimic TARnet \citep{shalit2017estimating}: it uses separated NNs for the two potential outcomes. We tried similar idea on IHDP dataset, and, as we show in Sec.~5.2, it has basically no merits for our method, because we have a principled balancing by our prior.

The encoder of CEVAE is more complex. To have post-treatment estimation, $q(\rt|\x)$ and $q(\y|\x,\rt)$ are added into the encoder. As a result, the ELBO of CEVAE has two additional likelihood terms corresponding to the two distributions. However, in our Intact-VAE, post-treatment estimation is given naturally by our standard encoder, thanks to the correspondence between our model and \eqref{eq:cate_by_bts}. 

% \paragraph{Justification}
% We give identification under specific and general assumptions in Theorem \ref{id_tion}, and consistent estimation in Corollary \ref{th:estimation}, given the consistency of VAE, which is widely assumed in practice. Moreover, we carefully distinguish assumptions on true generating process and assumptions on our model, and identify the assumptions that are important for causality. 

\section{Discussions and examples of the injective separation assumption}
\label{sec:p_inj}

We focus on univariate outcome on $\R$ which is the most practical case and the intuitions apply to more general types of outcomes. Then, $\vi$, the mapping between $\mu_0$ and $\mu_1$, is monotone, i.e, either increasing or decreasing. The increasing $\vi$ means, if a change of the value of $\x$ increases (decreases)  the outcome in the treatment group, then it is also the case for the controlled group. This is often true because the treatment does \textit{not} change the mechanism how the covariates affect the outcome, under the principle of ``independence of causal mechanisms (ICM)'' \citep{janzing2010causal}. The decreasing $\vi$ corresponds to another common interpretation when ICM does not hold. Now, the treatment does change the way covariates affect $\rvy$, but in a \textit{global} manner: it acts like a ``switch'' on the mechanism: the same change of $\x$ always has \textit{opposite} effects on the two treatment groups. 
% Also, if the outcome of control group is constant ($\mu_0$ is constant), then the difference of CATEs w.r.t different $\x$, i.e. the effect modification \citep[Sec.~4.1]{hernanCausalInferenceWhat2020}, 

We support the above reasoning by real world examples. First we give two examples where $\mu_0$ and $\mu_1$ are both monotone increasing. This, and also that both $\mu_t$ are monotone decreasing, are natural and sufficient conditions for increasing $\vi$, though not necessary. The first example is form Health. \citep{starling2019monotone} mentions that gestational age (length of pregnancy) has a monotone increasing effect on babies' birth weight, regardless of many other covariates. Thus, if we intervene on one of the other binary covariates (say, t = receive healthcare program or not), both $\mu_t$ should be monotone increasing in gestational age. The next example is from economics. \citep{gan2016efficiency} shows that job-matching probability is monotone increasing in market size. Then, we can imagine that, with t = receive training in job finding or not, the monotonicity is not changed. Intuitively, the examples corresponds to two common scenarios: the causal effects are accumulated though time (the first example), or the link between a covariate and the outcome is direct and/or strong (the second example). 

Examples for decreasing $\vi$ are rarer and the following is a bit deliberate. This example is also about babies' birth weight as the outcome. \citep{abrevaya2015estimating} shows that, with t = mother smokes or not and $\x$ = mother's age, the CATE $\tau(\vx)$ is monotone decreasing for $20 < \vx < 26$ (smoking decreases birth weight, and the absolute causal effect is larger for older mother). On the other hand, it is shown that birth weight slightly increases (by about 100g) in the same age range in a surveyed population \citep{wang2020changing}. Thus, it is convince that, smoking changes the the tendency of birth weight w.r.t mother's age from increasing to decreasing, and gives the large decreasing of birth weight (by about 300g) as its causal effect. This could be understood: the negative effects of smoking on mother's heath and in turn on birth weight are accumulated during the many years of smoking.

\section{Old lessons}
Nowhere in the main text refers this section, so you can omit it if not interested. However, if reading, you may gain insight of how we came to our final theoretical formulation.

\subsection{Identifiability of representation (is not enough)}
\label{notenough}

Here we explain that the model identifiability given in Theorem \ref{idmodel} alone is, albeit interesting, not enough for estimation of TEs. 

The importance of model identifiability can be seen clearly in the following corollary. That is, given $\rt=t$, the latent representation can be identified up to an invertible element-wise affine transformation. It can be easily understood by noting that, with the small noise and the injective $\vf$, the decoder degenerates to deterministic function and the latent representation $\rvz=\vf^{-1}(\rvy)$.
% in \eqref{model_param}, $\vg'=\vs'=\bm0$ and $\vz'= \vr'(\vx,\vy)=\vf'^{-1}(\vy)$ for all $\vx$. 

\begin{corollary}
\label{idrepre}
In Theorem \ref{idmodel}, let $\bm\sigma_{\rvy,t} = \bm0$, then $\rvz = \mathcal{A}_t(\rvz')$. 
%  and 2) $p_{\vtheta'}=p_{\vtheta}$
% \begin{equation}
% \label{eqidrepre}
%     \rvz = \mathcal{A}(\rvz')|t, \quad \vf' = \vf \circ \mathcal{A} |t % \diag(\va)\rvz' + \vb\coloneqq 
% \end{equation}
\end{corollary}

The good news is that, all the possible latent representations in our model are equivalent if we consider their independence relationships with any random variables, because any two of them are related by an \textit{invertible} mapping. However, the bad news is that, this holds only given $\rt=t$, while the definition of B/P-score involves both $t$. 
% Note that, P-score implies, but is not implied by, $\rvy\independent \rvz,\rvx|\PS,\rt$.

Consider how the \textit{recovered} $\rvz'$ would be used. For a control group ($t=0$) data point $(\vx, y, 0)$, the real challenge under finite sample is to predict the counterfactual outcome $y(1)$. Taking the observation, the encoder will output a posterior sample point $\vz'_0=\vf_0'^{-1}(y)=\mathcal{A}_0^{-1}(\vz_0)$ (with zero outcome noise, the encoder degenerates to a delta function: $q(\rvz|\vx, y, 0)=\delta(\rvz-\vf_0'^{-1}(y))$). Then, we should do \textit{counterfactual inference}, using decoder with counterfactual assignment $t=1$: $y_1'=\vf'_1(\vz'_0)=\vf_1\circ\mathcal{A}_1(\mathcal{A}_0^{-1}(\vz_0))$. This prediction can be arbitrary far from the truth $y(1)=\vf_1(\vz_0)$, due to the difference between $\mathcal{A}_1$ and $\mathcal{A}_0$. More concretely, this is because when learning the decoder, only the posterior sample of the treatment group ($t=1$) is fed to $\vf'_1$, and the posterior sample is different to the true value by the affine transformation $\mathcal{A}_1$, while it is $\mathcal{A}_0$ for $\vz_0'$.

Now we know what we need: $\mathcal{A}_0=\mathcal{A}_1$ so that the equivalence of independence holds unconditionally; and, there exists at least one representation that is indeed a B-score. Then, \textit{any} representation in our model will be a B-score. These indeed are what we have in \cite{wu2021beta}.
% The next subsection specifies a further assumption to achieve this.

\begin{proof}[Proof of Corollary 1]
In this proof, all equations and variables should condition on $t$, and we omit the conditioning in notation for convenience. 

When $\bm\sigma_{\rvy} = \bm0$, the decoder degenerates to a delta function: $p(\rvy|\rvz)=\delta(\rvy-\vf(\rvz))$, we have $\rvy=\vf(\rvz)$ and $\rvy'=\vf'(\rvz')$. 
% From the consistency of VAE, $\rvy$ should have the same support as $\rvy'$. 
For any $\vy$ in the common support of $\rvy,\rvy'$, there exist a \textit{unique} $\vz$ and a \textit{unique} $\vz'$ satisfy $\vy=\vf(\vz)=\vf'(\vz')$ (use injectivity). Substitute $\vy=\vf(\vz)$ into the l.h.s of \eqref{eq:class}, and $\vy=\vf'(\vz')$ into the r.h.s, 
% we have $\vz = \diag(\va)\vz' + \vb$. The relation is one-to-one for all $\vz$, 
so we get $\rvz=\mathcal{A}(\rvz')$. The result for $\vf$ follows. % For each pair $(\vz, \vz')$ satisfies $\rvz=\mathcal{A}(\rvz')$, we have $\vf'(\vz')=\vf(\vz)=\vf(\mathcal{A}(\vz'))$ from the above argument. So $\vf' = \vf \circ \mathcal{A}$.  % $p(\rvy|\rvz, \rt)=\delta(\rvy-\vf(\rvz))$
\end{proof}
A technical detail is that, $\vz, \vz'$ might not always be related by $\mathcal{A}$, because we used the \textit{common} support of $\rvy,\rvy'$ in the proof. Thus, the relation holds for partial supports of $\rvz, \rvz'$ correspond to the common support of $\rvy,\rvy'$. This problem disappears if we have the a consistent learning method (see Proposition \ref{consistency}).

\subsection{Balancing covariate and its two special cases}
Here we demonstrate part of our old, limited, theoretical formulation, and extract some insights from it.

The following definition was used in the old theory.
The importance of this definition is immediate from the definition of balancing score, that is, if a balancing \textit{covariate} is also a function of $\rvv$, then it is a balancing \textit{score}.

\begin{definition}[Balancing covariate]
Random variable $\rvx$ is a {\em balancing covariate} of random variable $\rvv$ if $\rt \independent \rvv|\rvx$. We also simply say $\rvx$ is \textit{balancing} (or \textit{non}-balancing if it does not satisfy this definition).
\end{definition}

Given that a balancing score of the true (hidden or not) confounder is sufficient for weak ignorability, a natural and interesting question is that, does a balancing covariate of the true confounder also satisfies weak ignorability? The answer is \textit{no}. To see why,
% and also to better understand the significance of Theorem \ref{id_tion1}, 
we give the next Proposition indicating that a balancing covariate of the true confounder might \textit{not} satisfy \textit{exchangeability}. 
% We also refer readers to Appendix 8.5 where 

% \begin{proof}
% For $\rvx=\bm\chi(\rvz)$ (injective proxy), 
% For $\rvz=\bm\omega(\rvx)$ (noiseless but non-injective proxy), let $f=e \circ \bm\omega$ and $b$ be identity in Theorem \ref{bscore}, again $f(\rvx)=f(\bm\beta(\rvz))=e(\rvz)$.
% \end{proof}

% More generally, we can prove the following result.
\begin{proposition}
Let $\rvx$ be a balancing covariate of $\rvv$. If $\rvv$ satisfies exchangeability \emph{and} $\rvy(t) \independent \rvx|\rvv,\rt$, then so does $\rvx$.
\end{proposition}
The proof will use the properties of conditional independence (Proposition \ref{indep_prop}).
\begin{proof}[Proof]
Let $\rvw\coloneqq \rvy(t)$ for convenience. We first write our assumptions in conditional independence, as \textit{A1}. $ \rt \independent \rvv|\rvx$ (balancing covariate), \textit{A2}. $\rvw \independent \rt| \rvv$ (exchangeability given $\rvv$), and \textit{A3}. $\rvw \independent \rvx|\rvv,\rt$.

Now, from \textit{A2} and \textit{A3}, using contraction, we have $\rvw \independent \rvx,\rt|\rvv$, then using weak union, we have $\rvw \independent \rt|\rvx,\rvv$. From this last independence and \textit{A1}, using contraction, we have $\rt \independent \rvv,\rvw|\rvx$. Then $\rt \independent \rvw|\rvx$ follows by decomposition.
\end{proof}

Given this proposition, we know assumptions
\begin{equation}
\label{old_ass}
    \begin{split}
        \text{i) }& \rvy(t) \independent \rt|\rvv \text{ (exchangeability given $\rvv$), } \\
        \text{ii) }& \rt \independent \rvv|\rvx \text{ ($\x$ is a balancing covariate of $\rvv$), and} \\
        \text{iii) }& \rvy \independent \rvx|\rvv,\rt
    \end{split}
\end{equation}
do not imply exchangeability given $\x$, thus seem to be reasonable. 
% , and our method can work under unobserved confounding ($\rvx$ might not satisfy ignorability). 
Note the independence $\rvy(t) \independent \rvx|\rvv,\rt$ assumed in the above proposition implies, but is not implied by, $\rvy \independent \rvx|\rvv,\rt$.
% From the decomposition rule of conditional independence, we have $\rvy(0), \rvy(1) \independent \rvx|\rvv,\rt \implies \forall t (\rvy(t) \independent \rvx|\rvv,\rt) \implies \forall t,\hat{t} (\rvy(t) \independent \rvx|\rvv,\rt=\hat{t})$, 4 independence in total. Using only two of them, and with the consistency of counterfactual, we have $\forall t (\rvy(t) \independent \rvx|\rvv,\rt=t) \implies \forall t (\rvy \independent \rvx|\rvv,\rt=t) \implies \rvy \independent \rvx|\rvv,\rt$. 
This is because, in general, $\rvy(0) \independent \rvx|\rvv,\rt=1$ and $\rvy(1) \independent \rvx|\rvv,\rt=0$ do not hold.
% , and thus $\rvy \independent \rvx|\rvv,\rt \centernot\implies \rvy(0), \rvy(1) \independent \rvx|\rvv,\rt$.

The assumptions in \eqref{old_ass} were assumed by our old theory, with $\rvv$ is hidden confounder $\rconf$ plus observed confounder $\x_c$. And also note that, iii) is the independence shared by Bt-score.

% is a different, and in fact a much \textit{weaker} one. 
% By the structural definition of counterfactuals \citep[eq (3.51)]{pearl2009causality}, $(\rvy(0), \rvy(1))$ are equivalent to the sum total of all exogenous variables (random sources) that can influence $\rvy$ through paths that \textit{avoid} $\rt$. [for the former, x should be independence of all exogenous variables that do not influence z and t, while for the latter, should also  be independence of all exogenous variables that influence both t and y].

% More formally, we have:
% [[
% \textcolor{red}{
% It would be interesting to prove one of below, or give counterexamples. We might need to use the above two independence assumptions.
% \begin{proposition} [What we are happy to see]
% If $\rvx$ satisfies (strong) ignorability then it must be a balancing covariate of $\rvv$ which satisfies strong ignorability. [Then our assumption is weaker than (strong) ignorability.]
% \end{proposition}
% \begin{proposition} [Make our method more trial]
% Assume $\rvv$ satisfies strong ignorability and let $\rvx$ be a balancing covariate of $\rvv$. Then $\rvx$ satisfies ignorability. [We are happy to see counterexamples, since they will be examples that our method can work without ignorability.]
% \end{proposition}
% }
% ]]

We examine two important special cases of balancing covariate, which provide further evidence that balancing covariate does not make the problem trivial.
% , one of those is noiseless proxy, which might \textit{not} satisfy \textit{positivity}. 

\begin{definition}[Noiseless proxy]
Random variable $\rvx$ is a noiseless proxy of random variable $\rvv$ if $\rvv$ is a function of $\rvx$ ($\rvv=\bm\omega(\rvx)$). 
\end{definition}
Noiseless proxy is a special case of balancing covariate because if $\rvx=\vx$ is given, we know $\vv=\bm\omega(\vx)$ and $\bm\omega$ is a deterministic function, then $p(\rvv|\rvx=\vx)=p(\rvv|\rvx=\vx,\rt)=\delta(\rvv-\bm\omega(\vx))$.
Also note that, a noiseless proxy always has higher dimensionality than $\rvv$, or at least the same. 

Intuitively, if the value of $\rvx$ is given, there is no further uncertainty about $\vv$, so the observation of $\vx$ may work equally well to adjust for confounding. But, as we will see soon, a noiseless proxy of the true confounder does \textit{not} satisfy positivity. % It is easy to see why noiseless proxy is more ideal. % a noiseless proxy of the true confounder satisfies ignorability, but \textit{not} positivity.

\begin{definition} [Injective proxy]
Random variable $\rvx$ is an injective proxy of random variable $\rvv$ if $\rvx$ is an injective function of $\rvv$ ($\rvx=\bm\chi(\rvv)$, $\bm\chi$ is injective). 
\end{definition}
Injective proxy is again a special case of noiseless proxy, since, by injectivity, $\rvv=\bm\chi^{-1}(\rvx)$, i.e. $\rvv$ is also a function of $\rvx$. 

Under this very special case, that is, if $\rvx$ is an injective proxy of the true confounder $\rvv$, we finally have $\rvx$ is a balancing score and satisfies weak ignorability, since $\rvx$ is a balancing covariate and a function of $\rvv$. To see this in another way, let $f=e \circ \bm\chi^{-1}$ and $\bm\beta=\bm\chi$ in Proposition \ref{bscore}, then $f(\rvx)=f(\bm\beta(\rvv))=e(\rvv)$. By weak ignorability of $\rvx$, \eqref{eq:id} has a simpler counterpart $\mu_t(\vx) = \E(\rvy(t)|\rvx=\vx) = \E(\rvy|\rvx=\vx,\rt=t)$. Thus, a naive regression of $\rvy$ on $(\rvx, \rt)$ will give a valid estimator of CATE and ATE.  %, that is, $\rvx$ is a balancing score. If $\rvx$ is an injective proxy of the true confounder, then it is a balancing score of the confounder. If $\rvx$ is an injective proxy of the true confounder, then it is a balancing score of the confounder. 

However, a noiseless but \textit{non}-injective proxy is \textit{not} a balancing score, in particular, positivity might \textit{not} hold. Here, a naive regression will not do. This is exactly because $\bm\omega$ is non-injective, hence multiple values of $\rvx$ that cause non-overlapped supports of $p(\rt=t|\rvx=\vx),t=0,1$ might be mapped to the same value of $\rvv$. An extreme example would be $\rt=\mathbb{I}(\rx>0),\rz=|\rx|$. We can see $p(\rt=t|\rx)$ are totally non-overlapped, but $\forall t,z \neq 0: p(\rt=t|\rz=z)=1/2$.

So far, so good. In the end, what is the problem of balancing covariate? Here it is. \textit{If} the we have the positivity of $\x$ ($p(\rt|\x)>0$ always), then, using the positivity and balancing to get $p(\conf|\vx)=p(\conf|\vx,\rt=t)$ for all $\vx$, we follow \eqref{eq:id},
\begin{equation}
\begin{split}
    \mu_t(\vx)&=\textstyle \int (\int p(y|\conf,\vx,t)ydy)p(\conf|\vx)d\conf \\
    &=\textstyle \int (\int p(y|\conf,\vx,t)ydy)p(\conf|\vx,\rt=t)d\conf \\
    &=\textstyle \int (\int p(y,\conf|\vx,t)d\conf)ydy=\E(\y|\vx,t).
% \vspace{-8mm}
\end{split}
\end{equation}
Naive estimator just works! Thus, \textit{if} $\x$ indeed was a balancing covariate of true confounder, we gave a better method than naive estimator only in the sense that it works without positivity of $\x$. It seems what our old theory really addressed was lack of positivity, another important issue in causal inference \citep{d2020overlap}, but not confounding.

This limited formulation, together with the great experimental performance of our method, motivated us to develop a much more general theory, that is, the theory based on B*-scores in the main text.

There are several lessons learned from the old formulation. First, there may exist cases that exchangeability given $\x$ fails to hold even when positivity of $\x$ holds, but the naive estimator still works. This is related to the fact that the conditional independence based on which balancing score/covariate are defined is not necessary for identification. And we should be able to find weaker but still sufficient conditions for identification, and Bt-score is an example. Second, balancing covariate assumption in \eqref{old_ass} is strong, though may not make a trivial problem. It basically means $\x$, only one of the observables, is sufficient information for treatment assignment. This inspires us to consider both $\x,\y$ in our theory, as in the Bt-score given by our posterior and encoder.

% [by learning a \textit{lower} dimensional representation, we overcome the possible lack of positivity of proxy.][cite dimension and positivity also CFR]

% [add some words on omitting t]

\section{Details and additional results for experiments}

\subsection{Synthetic data}
\label{sec:exp_syn_app}

We generate data following \eqref{art_model} with $\rz$, $\rvy$ 1-dimensional and $\rvx$ 3-dimensional. $\mu_i$ and $\sigma_i$ are randomly generated in range $(-0.2, 0.2)$ and $(0, 0.2)$, respectively. 

\begin{wrapfigure}{r}{0.4\textwidth}

\vspace{-.2in}
  \begin{center}
    \includegraphics[width=0.4\textwidth]{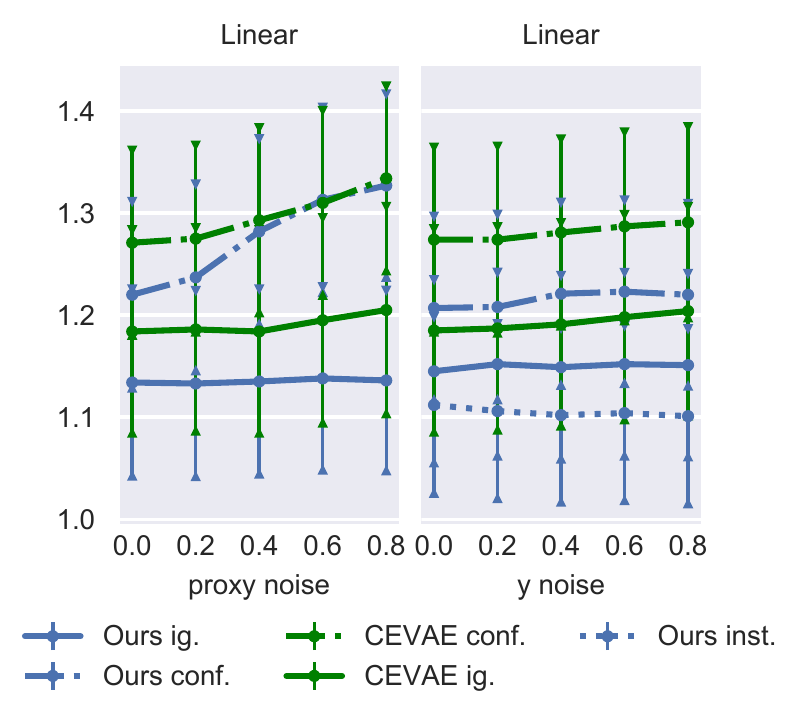}
  \end{center}
  \vspace{-.2in}
  
  \caption{\footnotesize{$\sqrt{\epsilon_{pehe}}$ on linear synthetic dataset. Error bar on 100 random models. We adjust one of $\alpha,\beta$ at a time. Results for ATE and post-treatment are similar.}}
\vspace{-.1in}
\label{lin_art}
\end{wrapfigure}

We adjust the outcome and proxy noise level by $\alpha,\beta$ respectively.
The output of $f$ is normalized by $C\coloneqq\Var_{\{\mathcal{D}|\rt=t\}}(f(\rz))$. This means we need to use $0 \leq \alpha < 1$ to have a reasonable level of noise on $\rvy$ (the scales of mean and variance are comparable). Similar reasoning applies to $\rz|\rvx$; outputs of $h,k$ have approximately the same range of values since the functions' coefficients are generated by the same weight initializer.

As shown in Figure \ref{lin_art}, our method again constantly outperforms CEVAE under linear outcome models. Interestingly, linear outcome models seem harder for both methods\footnote{Note that, after generating the outcomes and before the data is used, we normalize the distribution of ATE of the 100 generating models, so the errors on linear and nonlinear settings are basically comparable.}. While we did not dig into this point because this would be a digress from our purpose, we give two possible reasons: 1)  under similar noise levels, the observed outcome values under nonlinear outcome models might be more informative about the values of $\rz$, because nonlinear models are often steeper than linear models for many values of $\rz$; 2) the two true linear outcome models for $t=0,1$ are more similar, particularly when the two potential outcomes are in small and similar ranges, and it is harder to distinguish and learning the two outcome models.

You can find more plots for latent recovery at the end of the paper.

\subsection{IHDP} 
\label{sec:ihdp_app}

IHDP is based on an RCT where each data point represents a child with 25 features about their birth and mothers. \texttt{Race} is introduced as a confounder by artificially removing all treated children with nonwhite mothers. There are 747 subjects left in the dataset. The outcome is synthesized by taking the covariates (features excluding \texttt{Race}) as input, hence \textit{unconfoundedness} holds given the covariates. 

Following previous work, we split the dataset by 63:27:10 for training, validation, and testing.

% [Targeted Maximum Likelihood Estimation (TMLE) \citep{van2011targeted}
% Bayesian Additive Regression Trees (BART) \citep{chipman2010bart} and Causal Forests \citep{wager2018estimation}]

% \subsection{More latent plot on synthetic dataset}
% See next pages.

\subsection{Additional plots on synthetic datasets}
\label{sec:plots}
See last pages.

\section{Discussions}

% As we see in the balanced estimator \eqref{bestimator}, our representation $\z$ is in fact capable of \textit{counterfactual inference}: $\hat{t}$ can be different to factual $\rt=t$. Experiments on counterfactual generation, like those in \citep[CausalGAN]{kocaoglu2017causalgan} and \citep[CausalVAE]{yang2020causalvae}, are on the way. 

Since our method works without the recovery of either hidden confounder or true score distribution, we often cannot see apparent relationships between recovered latent representation and the true hidden confounder/scores. It would be nice to directly see the learned representation preserves causal properties, for example, by some causally-specialized metrics, e.g. \cite{suter2019robustly}. 

Despite the formal requirement in Theorem \ref{idmodel} of fixed distribution of noise on $\rvy$, inherited from \cite{khemakhem2020variational}, the experiments show evidence that our method can learn the outcome noise. We observed that, in most cases, allowing the noise distribution to be learned depending on $\rvz,\rt$ improves performance.  Theoretical analysis of this phenomenon is an interesting direction for future work. 

% We conjecture that, it is possible to extend model identifiability to conditional noise models $\vg_t(\z)$.
% And we expect that the noise on $\y$ can also be identified up to some eq.~class (or joint eq.~class together with $\vf$). In that case, the model identifiability may also be sufficient for causal inference, under some respective assumptions on true generating process (note the current \ref{ass:gen} in Theorem \ref{id_tion} also seems stronger than needed), and our current \ref{ass:gen_n} and \ref{ass:model_latent} in Theorem \ref{id_tion} can be \textit{relaxed} to large extent. Similarly to current $\vf$, we may have identification for a general class of noises. 

Also, our causal theory does not in principle require continuous latent distributions, though in Theorem \ref{idmodel}, differentiability of $\vf$ is inherited from iVAE. Given the fact that currently all nonlinear ICA based identifiability requires differentiable mapping between the latent and observables, directly based on it, theoretical extensions to \textit{discrete} latent variable would be challenging. However, what is essential for CATE identification is the \textit{same} transformation between true and recovered score distribution for both $t$, but the transformation needs \textit{not} to be affine, and, possibly, neither injective. This opens directions for future extensions, based not necessarily on nonlinear ICA.

% \bibliography{example_paper}
% \bibliographystyle{icml2021}

% \clearpage

% \onecolumn
\begin{figure}[h] 
    \vspace*{-0.1in}
    \centering
    \includegraphics[width=.9\textwidth]{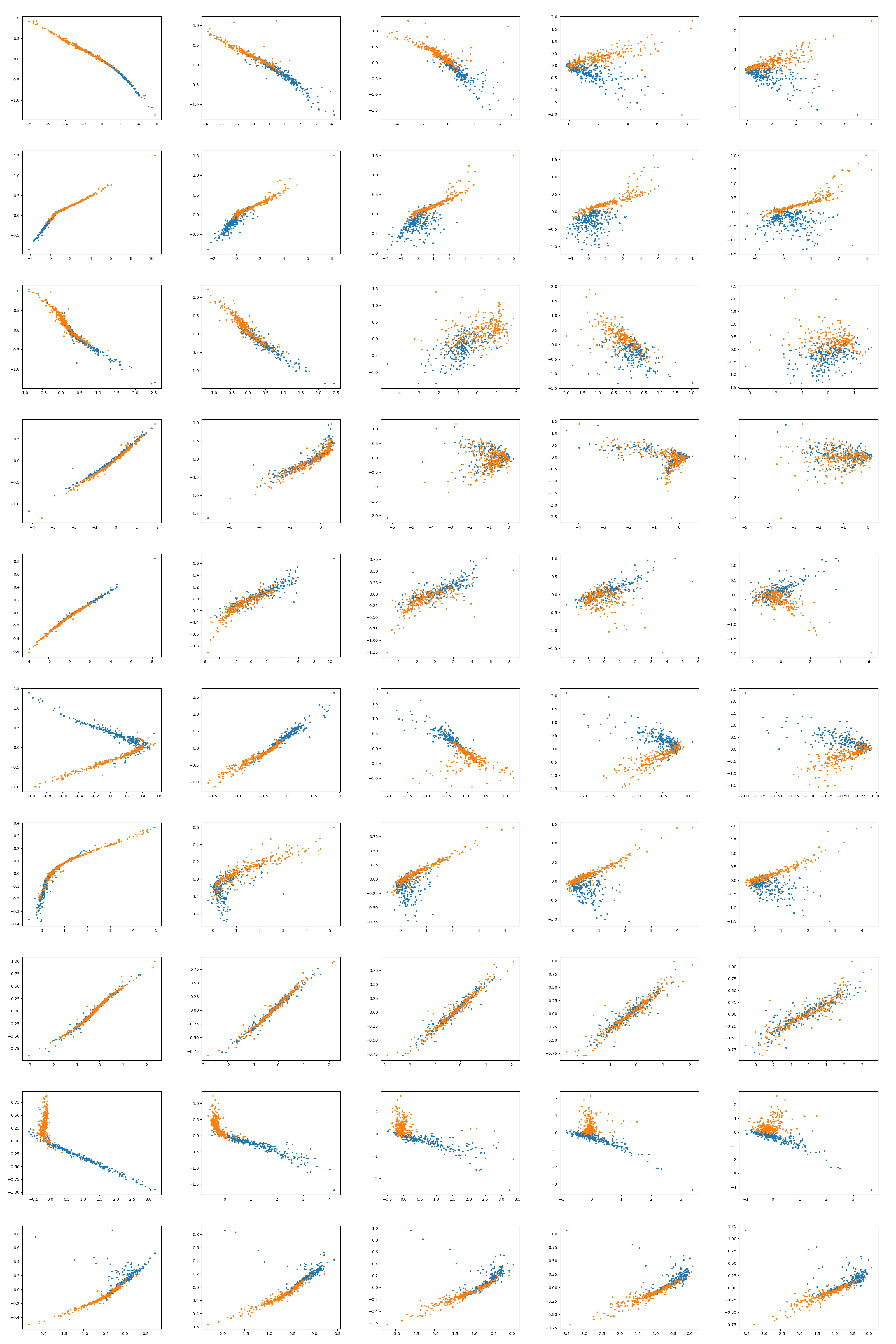}
    \vspace*{-0.1in}
    \caption{Plots of recovered-true latent under \textit{unobserved confounding}. Rows: first 10 nonlinear random models, columns: \textit{proxy} noise level.}
    \vspace*{-0.1in}
    % \label{dc}
\end{figure}

% \newpage

\begin{figure}[h] 
    \vspace*{-0.1in}
    \centering
    \includegraphics[width=.9\textwidth]{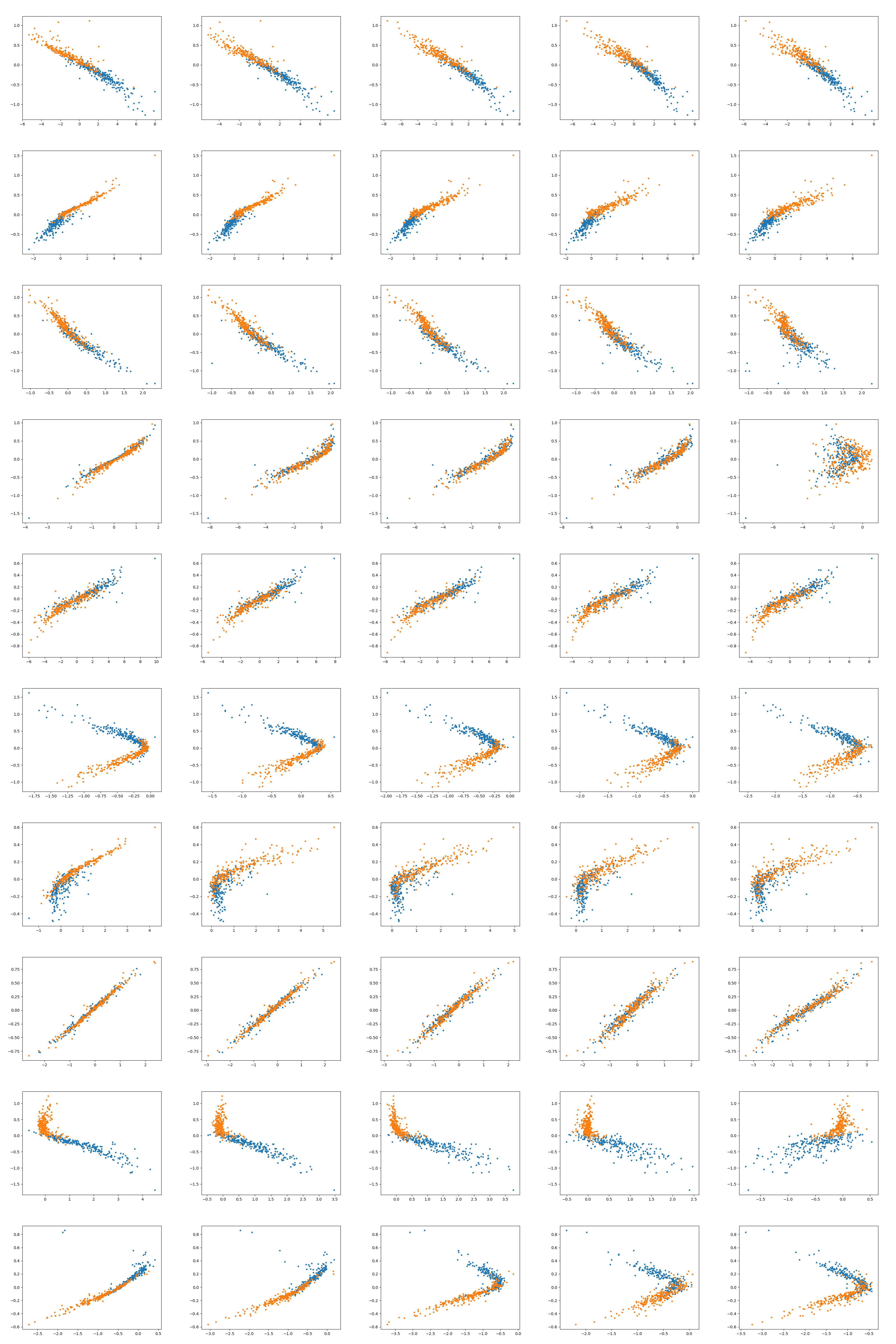}
    \vspace*{-0.1in}
    \caption{Plots of recovered-true latent under \textit{unobserved confounding}. Rows: first 10 nonlinear random models, columns: \textit{outcome} noise level.}
    \vspace*{-0.1in}
    % \label{dc}
\end{figure}

\begin{figure}[h] 
    \vspace*{-0.1in}
    \centering
    \includegraphics[width=.9\textwidth]{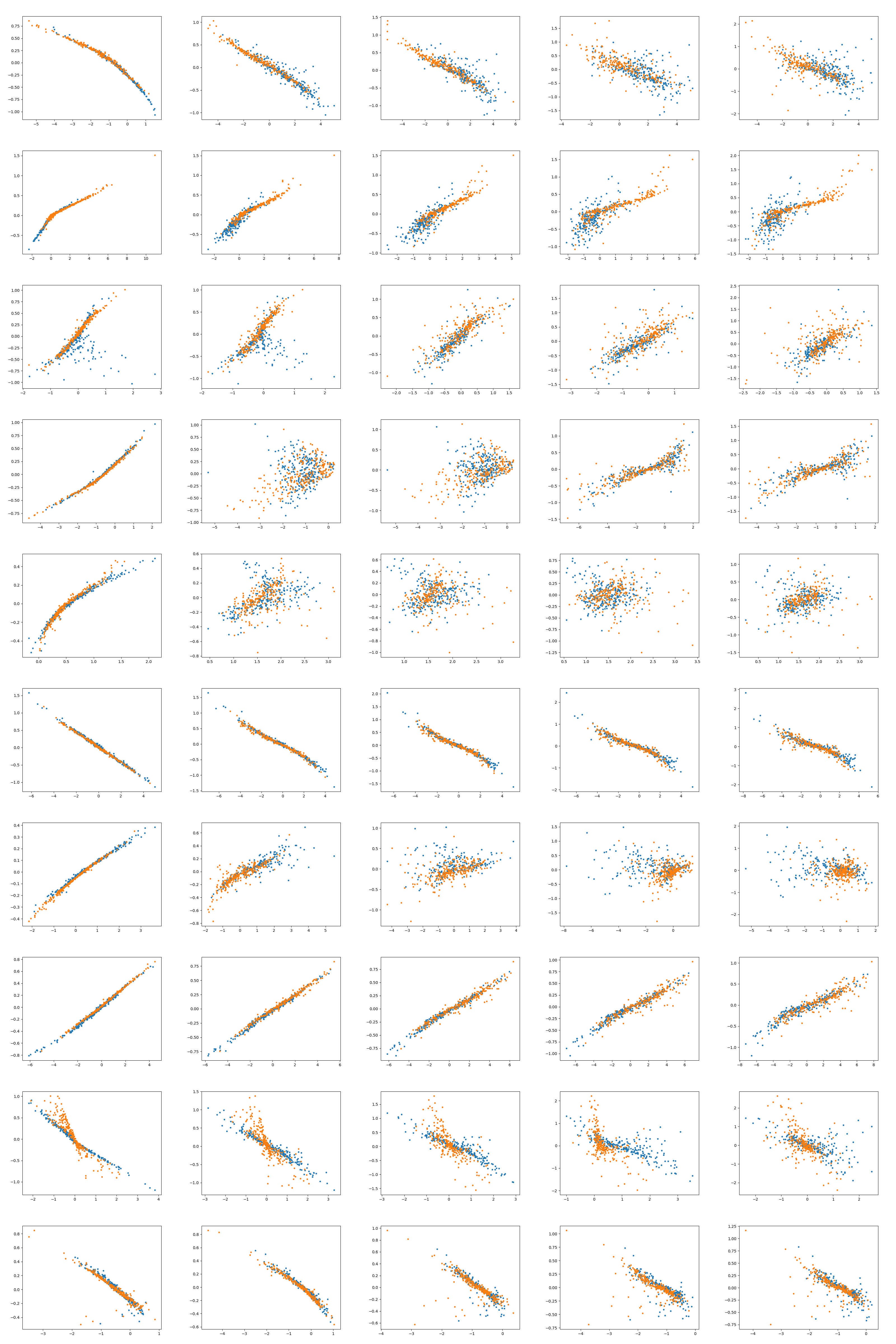}
    \vspace*{-0.1in}
    \caption{Plots of recovered-true latent when \textit{ignorability} holds. Rows: first 10 nonlinear random models, columns: \textit{proxy} noise level.}
    \vspace*{-0.1in}
    % \label{dc}
\end{figure}

\begin{figure}[h] 
    \vspace*{-0.1in}
    \centering
    \includegraphics[width=.9\textwidth]{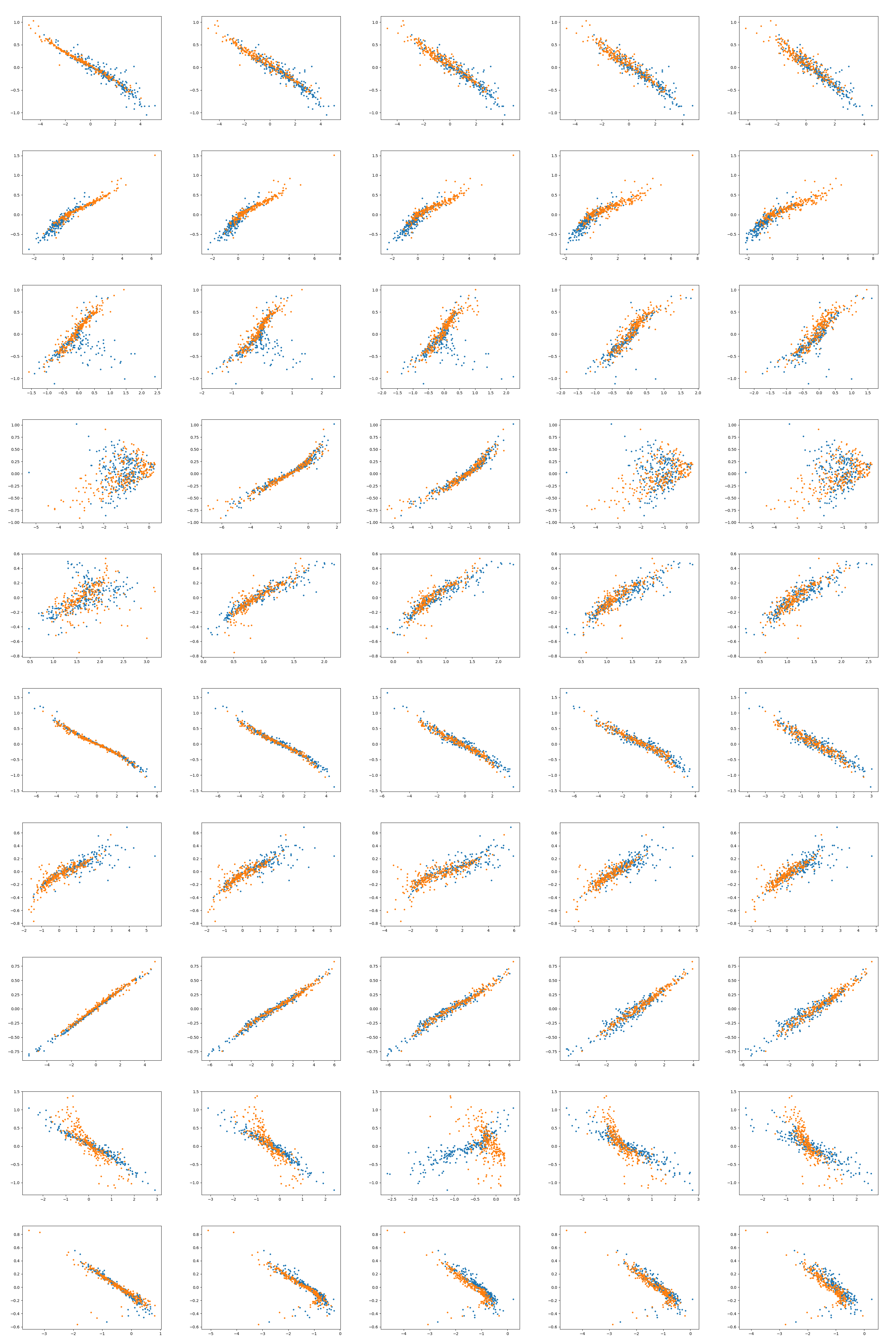}
    \vspace*{-0.1in}
    \caption{Plots of recovered-true latent when \textit{ignorability} holds. Rows: first 10 nonlinear random models, columns: \textit{outcome} noise level.}
    \vspace*{-0.1in}
    % \label{dc}
\end{figure}

\begin{figure}[h] 
    \vspace*{-0.1in}
    \centering
    \includegraphics[width=.9\textwidth]{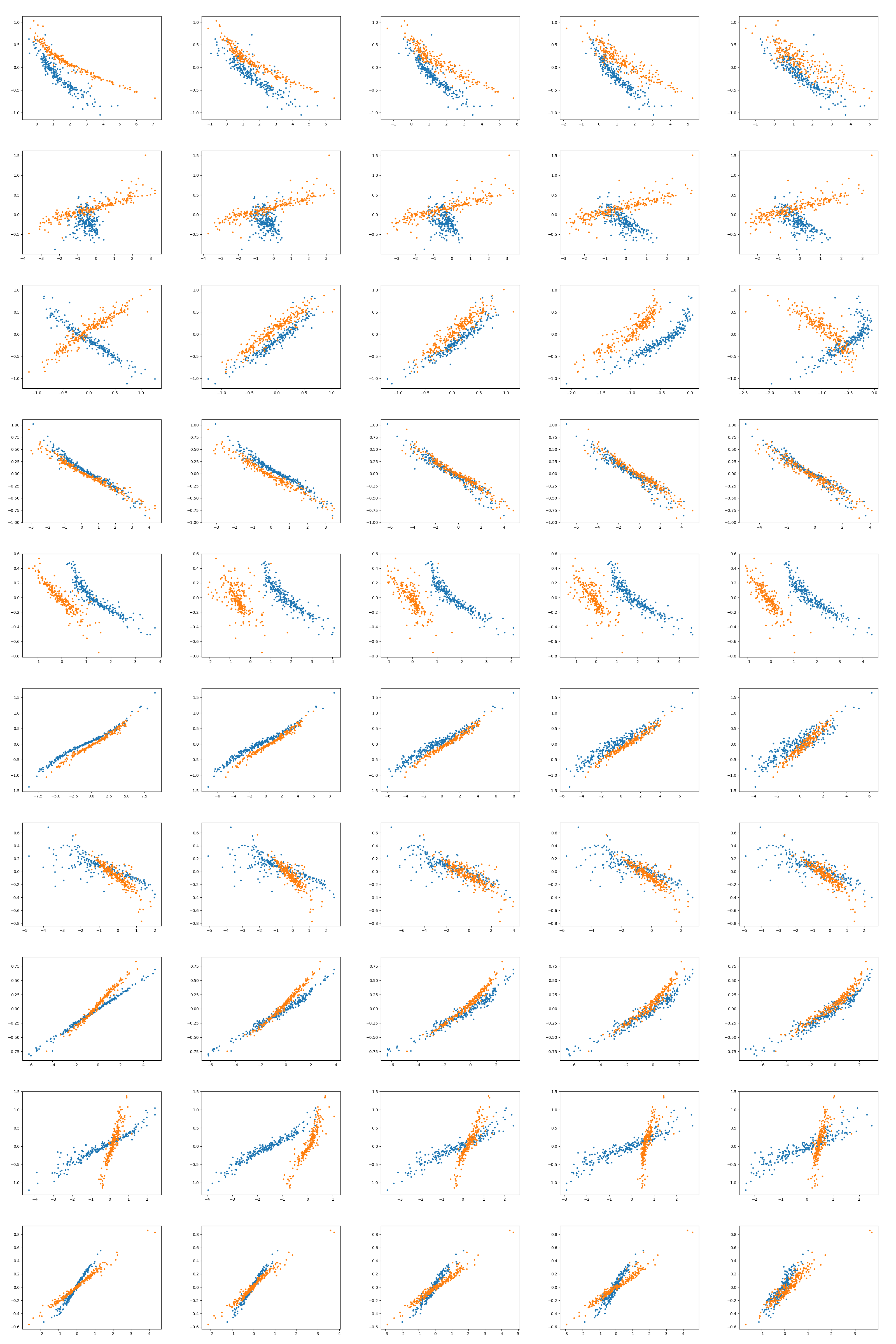}
    \vspace*{-0.1in}
    \caption{Plots of recovered-true latent when \textit{ignorability} holds. Conditional prior \textit{depends} on $t$. Rows: first 10 nonlinear random models, columns: \textit{outcome} noise level. Compare to the previous figure, we can see the transformations for $t=0,1$ are \textit{not} the same.}
    \vspace*{-0.1in}
    % \label{dc}
    \label{fig:bad_recover}
\end{figure}

\begin{figure}[h] 
    \vspace*{-0.1in}
    \centering
    \includegraphics[width=.9\textwidth]{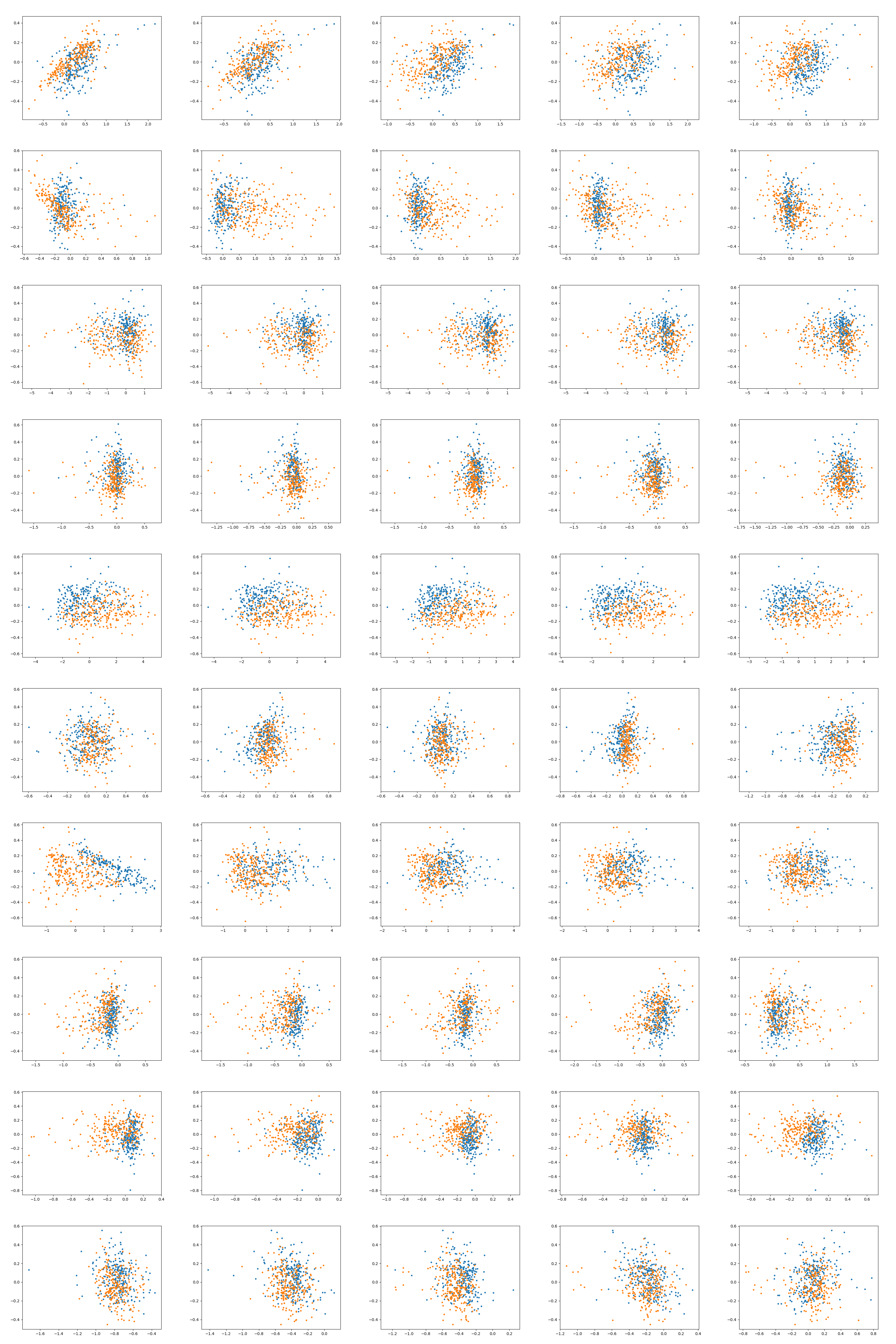}
    \vspace*{-0.1in}
    \caption{Plots of recovered-true latent on \textit{IVs}. Rows: first 10 nonlinear random models, columns: \textit{outcome} noise level.}
    \vspace*{-0.1in}
    % \label{dc}
\end{figure}

\end{document}

%% file: math_commands.tex
\usepackage{amsmath,amsfonts,bm}

\def\1{\bm{1}}

\def\rt{{\textnormal{t}}}

\def\rw{{\textnormal{w}}}
\def\rx{{\textnormal{x}}}
\def\ry{{\textnormal{y}}}
\def\rz{{\textnormal{z}}}

\def\rvc{{\mathbf{c}}}

\def\rve{{\mathbf{e}}}

\def\rvu{{\mathbf{i}}}

\def\rvt{{\mathbf{t}}}
\def\rvu{{\mathbf{u}}}
\def\rvv{{\mathbf{v}}}
\def\rvw{{\mathbf{w}}}
\def\rvx{{\mathbf{x}}}
\def\rvy{{\mathbf{y}}}
\def\rvz{{\mathbf{z}}}

\def\vtheta{{\bm{\theta}}}
\def\va{{\bm{a}}}
\def\vb{{\bm{b}}}
\def\vc{{\bm{c}}}

\def\vf{{\bm{f}}}
\def\vg{{\bm{g}}}
\def\vh{{\bm{h}}}
\def\vi{{\bm{i}}}
\def\vj{{\bm{j}}}
\def\vk{{\bm{k}}}

\def\vr{{\bm{r}}}
\def\vs{{\bm{s}}}
\def\vt{{\bm{t}}}
\def\vu{{\bm{u}}}
\def\vv{{\bm{v}}}

\def\vx{{\bm{x}}}
\def\vy{{\bm{y}}}
\def\vz{{\bm{z}}}

\def\mA{{\bm{A}}}

\def\mG{{\bm{G}}}

\def\mI{{\bm{I}}}
\def\mJ{{\bm{J}}}

\def\mL{{\bm{L}}}

\def\mO{{\bm{O}}}

\def\mR{{\bm{R}}}

\def\mX{{\bm{X}}}

\DeclareMathAlphabet{\mathsfit}{\encodingdefault}{\sfdefault}{m}{sl}
\SetMathAlphabet{\mathsfit}{bold}{\encodingdefault}{\sfdefault}{bx}{n}

\newcommand{\E}{\mathbb{E}}

\newcommand{\R}{\mathbb{R}}

\newcommand{\KL}{D_{\mathrm{KL}}}
\newcommand{\Var}{\mathrm{Var}}